%% file: main.tex
\documentclass[letterpaper,twocolumn,10pt]{article}
\usepackage[utf8]{inputenc}
\usepackage{usenix2019_v3}
\usepackage{cite}
\usepackage{algorithm}
\usepackage{algpseudocode}
\usepackage{amsfonts}
\usepackage{amsmath}
\usepackage{amsthm}
\usepackage{authblk}
\usepackage{mathtools}
\usepackage{bbm}
\usepackage{bm}
\usepackage{booktabs}
\usepackage{overpic}
\usepackage{multirow}
\usepackage{enumitem}
\usepackage{appendix}

\usepackage{graphicx}
\usepackage[textsize=tiny,backgroundcolor=yellow]{todonotes}
\usepackage{xcolor}
\usepackage{pifont}

\DeclareMathOperator*{\argmax}{arg\,max}

\newcommand{\oracle}{\mathcal{O}}
\newcommand{\guess}{\hat{\oracle}} 
\newcommand{\exalg}{\mathcal{A}}
\newcommand{\ddomain}{\mathcal{X}}
\newcommand{\drange}{\mathcal{Y}}
\newcommand{\distr}{\mathcal{D}}
\newcommand{\relu}{\text{ReLU}}

\makeatletter
\newtheorem*{rep@theorem}{\rep@title}
\newcommand{\newreptheorem}[2]{%
\newenvironment{rep#1}[1]{%
 \def\rep@title{#2 \ref{##1}}%
 \begin{rep@theorem}}%
 {\end{rep@theorem}}}
\makeatother

\newtheorem{theorem}{Theorem}
\newreptheorem{theorem}{Theorem}
\newtheorem{lemma}{Lemma}
\newreptheorem{lemma}{Lemma}
\newtheorem{definition}{Definition}[section]

\begin{document}

\title{High Accuracy and High Fidelity Extraction of Neural Networks}

\author[$\dagger, *$]{Matthew Jagielski}
\author[*]{Nicholas Carlini}
\author[*]{David Berthelot}
\author[*]{Alex Kurakin}
\author[*]{Nicolas Papernot}

\affil[$\dagger$] {Northeastern University}
\affil[*] {Google Research}

\maketitle
\thispagestyle{plain}
\pagestyle{plain}

\begin{abstract}
In a model extraction attack, an adversary steals a copy of a remotely deployed machine learning model,
given oracle prediction access.
We taxonomize model 
extraction attacks around two objectives: 
\emph{accuracy}, i.e., performing well on the underlying learning task, and 
\emph{fidelity}, i.e., matching the predictions of the remote victim classifier on any input.

To extract a high-accuracy model,
we develop a learning-based attack exploiting the victim to 
supervise the training of an extracted model.
Through analytical and empirical arguments, we then explain the inherent limitations that prevent any learning-based strategy from extracting a truly high-fidelity model---i.e., extracting a functionally-equivalent model whose predictions are identical to those of the victim model on all possible inputs.
Addressing these limitations, 
we expand on prior work to develop the first practical functionally-equivalent
extraction attack for direct extraction (i.e., without training) of a model's
weights.

We perform experiments both on academic datasets
and a state-of-the-art image classifier trained with 1 billion proprietary images.
In addition to broadening the scope of model extraction research, our work demonstrates the practicality of model extraction attacks against production-grade systems.

\end{abstract}

\section{Introduction}

Machine learning, and neural networks in particular, are  widely
deployed in industry
settings. Models are often deployed as prediction services or otherwise exposed
to potential adversaries.
Despite this fact, the trained models themselves are often
proprietary and are closely guarded.

There are two reasons models are often seen as sensitive.
First, they are expensive to obtain.
Not only is it expensive to train the final model \cite{strubell2019energy}
(e.g., Google recently trained a model with 340 million
parameters on hardware costing 61,000 USD \emph{per training run}~\cite{yang2019xlnet}),
performing the work to identify the optimal set of model architecture,
training algorithm, and hyper-parameters often eclipses the cost of training
the final model.
Further, training these models 
also requires investing in expensive collection process to
obtain the training datasets necessary to 
obtain an accurate classifier \cite{halevy2009unreasonable,deng2009imagenet,sutskever2014sequence,van2016wavenet}.
Second, there are security \cite{papernot2017practical, lowd2005adversarial} and privacy \cite{shokri2017membership, salem2018ml} concerns for revealing trained
models to potential adversaries.

Concerningly, prior work found that an adversary with query access to a model can steal the
model to obtain a copy that largely agrees with the remote victim models  \cite{lowd2005adversarial, tramer2016stealing, orekondy2019knockoff, DBLP:journals/corr/abs-1811-02054, oh2017towards, DBLP:journals/corr/abs-1905-09165, correia2018copycat}. 
These extraction attacks are therefore important to consider.

In this paper, we systematize the space of model extraction around two adversarial objectives: \textit{accuracy} and \textit{fidelity}. Accuracy measures the correctness of predictions made by the extracted model on the test distribution. Fidelity, in contrast, measures the general agreement between the extracted and victim models on any input. Both of these objectives are desirable, but they are in conflict for imperfect victim models: a high-fidelity extraction should replicate the errors of the victim, whereas a high-accuracy model should instead try to make an accurate prediction.
At the high-fidelity limit is \textit{functionally-equivalent} model extraction: the two models agree on all inputs, both on and off the underlying data distribution.

While most prior work considers accuracy~\cite{tramer2016stealing,papernot2017practical,DBLP:journals/corr/abs-1811-02054}, 
we argue that fidelity is often equally important.
When using model extraction to mount black-box adversarial
example attacks \cite{papernot2017practical}, fidelity ensures the attack is more effective because more adversarial examples transfer from the extracted model to the victim. Membership inference~\cite{shokri2017membership,salem2018ml} benefits from the extracted model closely replicating the confidence of predictions made by the victim. Finally, a functionally-equivalent extraction enables the adversary to inspect whether internal representations reveal unintended attributes of the input---that are statistically uncorrelated with the training objective, enabling the adversary to benefit from overlearning~\cite{song2019overlearning}.

We design one attack for each objective.
First, a \textit{learning-based attack}, which uses the victim to generate labels for training the extracted model. While existing techniques already achieve high accuracy, our attacks are $16\times$ more query-efficient and scale to larger models.
We perform experiments that surface inherent limitations of learning-based extraction attacks and argue that learning-based strategies are ill-suited to achieve high-fidelity extraction.
Then, we develop the first practical \textit{functionally-equivalent attack}, which directly recovers
a two-layer neural network's weights exactly given access to double-precision model inference.
Compared to prior work, which required a high-precision
power side-channel~\cite{hong2018security} or access to model gradients~\cite{milli2018model}, 
our attack only requires input-output
access to the model, while simultaneously scaling to larger networks than either of the prior methods.

We make the following contributions:
\begin{itemize}
    \item We taxonomize the space of model extraction attacks by exploring the objective of \textit{accuracy} and \textit{fidelity}.
    \item We improve the query efficiency of learning attacks for accuracy extraction
    and make them practical for millions-of-parameter models trained on
    billions of images.
    \item We achieve high-fidelity extraction by developing the first
    practical functionally-equivalent
    model extraction.
    \item We mix the proposed methods to obtain a hybrid
    method which improves both accuracy and fidelity extraction.
\end{itemize}


\section{Preliminaries}
\label{sec:prelim}
We consider classifiers with domain $\ddomain\subseteq\mathbb{R}^d$ and range $\drange\subseteq\mathbb{R}^K$; the output of the classifier is a distribution over $K$ class labels. The class assigned to an input $x$ by a classifier $f$ is $\argmax_{i\in[K]}f(x)_i$ (for $n\in\mathbb{Z}$, we write $[n]=\lbrace 1, 2, \ldots n\rbrace$). In order to satisfy the constraint that a classifier's output is a distribution, a \emph{softmax} $\sigma(\cdot)$ is typically applied to the output of an arbitrary function $f_L: \ddomain\rightarrow\mathbb{R}^K$:
\[
\sigma(f_L(x))_i = \frac{\exp(f_L(x)_i)}{\sum_j\exp(f_L(x)_j)}.
\]
We call the function $f_L(\cdot)$ the \emph{logit function} for a classifier $f$. To convert a class label into a probability vector, it is common to use \emph{one-hot encoding}: for a value $j\in[K]$, the one-hot encoding $OH(j; K)$ is a vector in $\mathbb{R}^K$ with $OH(j; K)_i=\mathbbm{1}(i = j)$---that is, it is 1 only at index $j$, and 0 elsewhere.

Model extraction concerns reproducing a victim model, or oracle, which we write $\oracle: \ddomain\rightarrow\drange$. The model extraction adversary will run an extraction algorithm $\exalg(\oracle)$, which outputs the extracted model $\guess$. We will sometimes parameterize the oracle (resp. extracted model) as $\oracle_\theta$ (resp. $\guess_\theta$) to denote that it has model parameters $\theta$---we will omit this when unnecessary or apparent from context.

In this work, we consider $\oracle$ and $\guess$ to both be neural networks. A neural network is a sequence of operations---alternatingly applying linear operations and non-linear operations---a pair of linear and non-linear operations is called a \emph{layer}. Each linear operation projects onto some space $\mathbb{R}^h$---the dimensionality $h$ of this space is referred to as the \emph{width} of the layer. The number of layers is the \emph{depth} of the network. The non-linear operations are typically fixed, while the linear operations have parameters which are learned during training. The function computed by layer $i$, $f_i(a)$, is therefore computed as $f_i(a)=g_i(A^{(i)}a + B^{(i)})$, where $g_i$ is the $i$th non-linear function, and $A^{(i)}, B^{(i)}$ are the parameters of layer $i$ ($A^{(i)}$ is the weights, $B^{(i)}$ the biases). A common choice of activation is the rectified linear unit, or ReLU, which sets $\relu(x)=\max(0, x)$.  Introduced to improve the convergence of optimization when training neural networks, the ReLU activation has established itself as an effective default choice for practitioners~\cite{nair2010rectified}. Thus, we consider primarily ReLU networks in this work. 

The network structure described here is called \emph{fully connected} because each linear operation ``connects" every input node to every output node. In many domains, such as computer vision, this is more structure than necessary. A neuron computing edge detection, for example, only needs to use information from a small region of the image. \emph{Convolutional networks} were developed to combat this inefficiency---the linear functions become filters, which are still linear, but are only applied to a small (e.g., 3x3 or 5x5) window of the input. They are applied  to every window using the same weights, making convolutions require far fewer parameters than fully connected networks.

Neural networks are trained by \emph{empirical risk minimization}. Given a dataset of $n$ samples $\mathcal{D}=\lbrace x_i, y_i\rbrace_{i=1}^n\subseteq \ddomain\times\drange$, training involves minimizing a loss function $L$ on the dataset with respect to the parameters of the network $f$. A common loss function is the cross-entropy loss $H$ for a sample $(x,y)$: $H(y, f(x)) = -\sum_{k\in[K]} y_k\log(f(x)_k)$, where $y$ is the probability (or one-hot) vector for the true class. The cross-entropy loss on the full dataset is then
\[
L(\mathcal{D}; f) = \frac{1}{n}\sum_{i=1}^n H(y_i, f(x_i)) = -\frac{1}{n}\sum_{i=1}^n \sum_{k\in[K]} y_k\log(f(x)_k).
\]

The loss is minimized with some form of gradient descent, often stochastic gradient descent (SGD). In SGD, gradients of parameters $\theta$ are computed over a randomly sampled batch $B$, averaged, and scaled by a learning rate $\eta$:
\[
\theta_{t+1} = \theta_t - \frac{\eta}{|B|}\sum_{i\in B}\nabla_{\theta} H(y_i, f(x_i)).
\]
Other optimizers ~\cite{nesterov1983method, duchi2011adaptive, kingma2014adam} use gradient statistics to reduce the variance of updates which can result in better performance.

A less common setting, but one which is important for our work, is when the target values $y$ which are used to train the network are not one-hot values, but are probability vectors output by a different model $g(x)$. When training using the dataset $\mathcal{D}_g = \lbrace x_i, g(x_i)^{1/T}\rbrace_{i=1}^n$, we say the trained model is \emph{distilled} from $g$ with \emph{temperature} $T$, referring to the process of distillation introduced in Hinton \emph{et al.}~\cite{hinton2015distilling}. Note that the values of $g(x_i)^{1/T}$ are always scaled to sum to 1.

\section{Taxonomy of Threat Models}
\label{sec:taxonomy}

We now address the spectrum of adversaries interested in extracting neural networks. As illustrated in Table~\ref{tab:priorwork}, we taxonomize the space of possible adversaries around two overarching goals---\textit{theft} and \textit{reconnaissance}. We detail why extraction is not always practically realizable by constructing models that are impossible to extract, or require a large number of queries to extract. We conclude our threat model with a discussion of how adversarial capabilities (e.g., prior knowledge of model architecture or information returned by queries) affect the strategies an adversary may consider.

\begin{table*}[]
\footnotesize{}
    \centering
    \begin{tabular}{lllll}
        \toprule
        Attack & Type & Model type & Goal & Query Output \\
        ~ & ~ & ~ & ~ & ~ \\
        \midrule
        Lowd \& Meek~\cite{lowd2005adversarial} & Direct Recovery & LM & Functionally Equivalent & Labels \\
        Tramer \emph{et al.}~\cite{tramer2016stealing} & (Active) Learning & LM, NN & Task Accuracy, Fidelity & Probabilities, labels \\
        Tramer \emph{et al.}~\cite{tramer2016stealing} & Path finding & DT & Functionally Equivalent & Probabilities, labels \\
        Milli \emph{et al.}~\cite{milli2018model} (theoretical) & Direct Recovery & NN (2 layer) & Functionally Equivalent & Gradients, logits \\
        Milli \emph{et al.}~\cite{milli2018model} & Learning & LM, NN & Task Accuracy & Gradients \\
        Pal \emph{et al.}~\cite{DBLP:journals/corr/abs-1905-09165} & Active learning & NN & Fidelity & Probabilities, labels \\
        Chandrasekharan \emph{et al.}~\cite{DBLP:journals/corr/abs-1811-02054} & Active learning & LM & Functionally Equivalent & Labels \\
        Copycat CNN~\cite{correia2018copycat} & Learning & CNN & Task Accuracy, Fidelity & Labels \\
        Papernot \emph{et al.}~\cite{papernot2017practical} & Active learning & NN & Fidelity & Labels \\
        CSI NN~\cite{batina2018csi} & Direct Recovery & NN & Functionally Equivalent & Power Side Channel \\
        Knockoff Nets~\cite{orekondy2019knockoff} & Learning & NN & Task Accuracy & Probabilities \\
        \midrule
        Functionally equivalent (this work) & Direct Recovery & NN (2 layer) & Functionally Equivalent & Probabilities, logits \\
        Efficient learning (this work) & Learning & NN & Task Accuracy, Fidelity & Probabilities \\
        \bottomrule
    \end{tabular}
    \vspace{.5em}
    \caption{Existing Model Extraction Attacks. Model types are abbreviated: LM = Linear Model, NN = Neural Network, DT = Decision Tree, CNN = Convolutional Neural Network.}
    \label{tab:priorwork}

\end{table*}

\subsection{Adversarial Motivations}

Model extraction attacks target the \textit{confidentiality} of a victim model
deployed on a remote service. A model refers here to both the architecture and its parameters. Architectural details include the learning hypothesis (i.e., neural network in our case) and corresponding details (e.g., number of layers and activation functions for neural networks). Parameter values are the result of training.

First, we consider \emph{theft} adversaries, motivated by economic incentives.
Generally, the defender went through an expensive process to design the model's architecture and train it to set parameter values.  Here, the model can be viewed as intellectual property that the adversary is trying to steal. A line of work has in fact referred to this as ``model stealing''~\cite{tramer2016stealing}.

In the latter class of attacks, the adversary is performing \textit{reconnaissance} to later mount attacks targeting other security properties of the learning system: e.g., its integrity with adversarial examples~\cite{papernot2017practical}, or privacy with training data membership inference~\cite{shokri2017membership,salem2018ml}. Model extraction enables an adversary previously operating in a black-box threat model to mount attacks against the extracted model in a white-box threat model. The adversary has---by design---access to the extracted model's parameters.  In the limit, this adversary would expect to extract an \textit{exact} copy of the oracle.

\begin{figure}
    \centering
    \includegraphics[width=.8\linewidth]{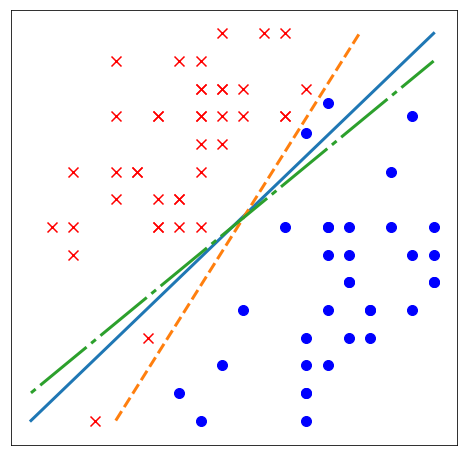}
    \caption{Illustrating fidelity vs. accuracy. The solid blue line is the oracle; functionally equivalent extraction recovers this exactly. The green dash-dot line achieves high fidelity: it matches the oracle on all data points. The orange dashed line achieves perfect accuracy: it classifies all points correctly. 
    }
    \label{fig:fid_vs_acc}
\end{figure}

The goal of \textbf{exact extraction} is to produce $\guess_\theta = \oracle_\theta$, so that the model's architecture and all of its weights are identical to the oracle. This definition is purely a strawman---it is the strongest possible attack, but it is fundamentally impossible for many classes of neural networks, including ReLU networks, because any individual model belongs to a large equivalence class of networks which are indistinguishable from input-output behavior. For example, we can scale an arbitrary neuron's input weights and biases by some $c>0$, and scale its output weights and biases by $c^{-1}$; the resulting model's behavior is unchanged. Alternatively, in any intermediate layer of a ReLU network, we may also add a \emph{dead neuron} which never contributes to the output, or might permute the (arbitrary) order of neurons internally. Given access to input-output behavior, the best we can do is identify the equivalence class the oracle belongs to.

\subsection{Adversarial Goals}

This perspective yields a natural spectrum of realistic adversarial goals characterizing decreasingly precise extractions.

\vspace{-4mm}

\paragraph{\textbf{Functionally Equivalent Extraction}}
The goal of functionally equivalent extraction is to construct an $\guess$ such that $\forall x\in\ddomain$, $\guess(x)=\oracle(x)$. This is a tractable weakening of the exact extraction definition from earlier---it is the hardest possible goal using only input-output pairs. The adversary obtains a member of the oracle's equivalence class. This goal enables a number of downstream attacks, including those involving inspection of the model's internal representations like overlearning~\cite{song2019overlearning}, to operate in the white-box threat model.

\vspace{-4mm}

\paragraph{\textbf{Fidelity Extraction}}
Given some target distribution $\distr_F$ over $\ddomain$, and goal similarity function $S(p_1, p_2)$, the goal of fidelity extraction is to construct an $\guess$ that maximizes $\Pr_{x\sim\distr_F}\left[S(\guess(x), \oracle(x))\right]$. In this work, we consider only \emph{label agreement}, where $S(p_1, p_2)=\mathbbm{1}(\argmax(p_1)=\argmax(p_2))$; we leave exploration of other similarity functions to future work. 

A natural distribution of interest $\distr_F$ is the data distribution itself---the adversary wants to make sure the mistakes and correct labels are the same between the two models. A reconnaissance attack for constructing adversarial examples would care about a perturbed data distribution; mistakes might be more important to the adversary in this setting. Membership inference would use the natural data distribution, including any outliers. These distributions tend to be concentrated on a low-dimension manifold of $\ddomain$, making fidelity extraction significantly easier than functionally equivalent extraction. Indeed, functionally equivalent extraction achieves a perfect fidelity of 1 on \emph{all distributions} and \emph{all similarity functions}.

\vspace{-4mm}

\paragraph{\textbf{Task Accuracy Extraction}}
For the true task distribution $\distr_A$ over $\ddomain\times\drange$, the goal of task accuracy extraction is to construct an $\guess$ maximizing $\Pr_{(x, y) \sim\distr_A}\left[\argmax(\guess(x))=y\right]$. This goal is to match (or exceed) the accuracy of the target model, which is the easiest goal to consider in this taxonomy (because it doesn't need to match the mistakes of $\oracle$).

\vspace{-4mm}

\paragraph{\textbf{Existing Attacks}}
In Table~\ref{tab:priorwork}, we fit previous model extraction work into this taxonomy, as well as discuss their techniques.
Functionally equivalent extraction has been considered for linear models~\cite{lowd2005adversarial, DBLP:journals/corr/abs-1811-02054}, decision trees~\cite{tramer2016stealing}, both given probabilities, and neural networks~\cite{milli2018model, batina2018csi}, given extra access.
%
%
%
Task accuracy extraction has been considered for linear models~\cite{tramer2016stealing} and neural networks~\cite{milli2018model, correia2018copycat, orekondy2019knockoff}, and fidelity extraction has also been considered for linear models~\cite{tramer2016stealing} and neural networks~\cite{DBLP:journals/corr/abs-1905-09165, papernot2017practical}. 
Notably, functionally equivalent attacks require model-specific techniques, while task accuracy and fidelity typically use generic learning-based approaches.

\subsection{Model Extraction is Hard}
\label{ssec:extracthard}
Before we consider adversarial capabilities in Section~\ref{ssec:adv-capabilities} and potential corresponding approaches to model extraction, we must understand how successful we can hope to be. Here, we present arguments that will serve to bound our expectations. First, we will identify some limitations of functionally equivalent extraction by constructing networks which require arbitrarily many queries to extract. Second, we will present another class of networks that cannot be extracted with fidelity without querying a number of times exponential in its depth. We provide intuition in this section and later prove these statements in Appendix~\ref{app:hardness_formal}.

\textbf{Exponential hardness of functionally equivalent attacks.}
\label{sssc:func_hard}
In order to show that functionally equivalent extraction is intractable in the worst case, we construct of a class of neural networks that are hard to extract without making exponentially many queries in the network's width.

\begin{theorem}
\label{thm:rectangle}
There exists a class of width $3k$ and depth 2 neural networks on domain $[0, 1]^d$ (with precision $p$ numbers) with $d\ge k$ that require, given logit access to the networks, $\Theta(p^k)$ queries to extract.
\end{theorem}

The precision $p$ is the number of possible values a feature can take from $[0, 1]$. In images with 8-bit pixels, we have $p=256$. The intuition for this theorem is that a width $3k$ network can implement a function that returns a non-zero value on at most a $p^{-k}$ fraction of the space. In the worst case, $p^{k}$ queries are necessary to find this fraction of the space.

Note that this result assumes the adversary can only observe the input-output behavior of the oracle. If this assumption is broken then functionally equivalent extraction becomes practical. For example, Batina \emph{et al.}~\cite{batina2018csi} perform functionally equivalent extraction by performing a side channel attack (specifically, differential power analysis~\cite{kocher1999differential}) on a microprocessor evaluating the neural network.

We also observe in Theorem~\ref{thm:subsetsum} that, given white-box access to two neural networks, it is NP-hard in general to test if they are functionally equivalent. We do this by constructing two networks that differ only in coordinates satisfying a subset sum instance. Then testing functional equivalence for these networks is as hard as finding the satisfying subset.

\begin{theorem}[Informal]
\label{thm:subsetsum}
Given their weights, it is NP-hard to test whether two neural networks are functionally equivalent.
\end{theorem}

Any attack which can claim to perform functionally equivalent extraction efficiently (both in number of queries used and in running time) must make some assumptions to avoid these pathologies. In Section~\ref{sec:dna}, we will present and discuss the assumptions of a functionally equivalent extraction attack for two-layer neural network models.

\paragraph{Learning approaches struggle with fidelity.}
\label{sssc:rand_hard}
A final difficulty for model extraction comes from recent work in learnability~\cite{DBLP:journals/corr/abs-1904-03866}. Das \emph{et al.} prove that, for deep random networks with input dimension $d$ and depth $h$, model extraction approaches that can be written as Statistical Query (SQ) learning algorithms require $\exp(O(h))$ samples for fidelity extraction. SQ algorithms are a restricted form of learning algorithm which only access the data with noisy aggregate statistics; many learning algorithms, such as (stochastic) gradient descent and PCA, are examples. As a result, most learning-based approaches to model extraction will inherit this inefficiency. A sample-efficient approach therefore must either make assumptions about the model to be extracted (to distinguish it from a deep random network), or must access its dataset without statistical queries.

\begin{theorem}[Informal~\cite{DBLP:journals/corr/abs-1904-03866}]
Random networks with domain $\lbrace 0, 1\rbrace^d$ and range $\lbrace 0, 1\rbrace$ and depth $h$ require $\exp(O(h))$ samples to learn in the SQ learning model.
\end{theorem}

\subsection{Adversarial Capabilities}
\label{ssec:adv-capabilities}

We organize an adversary's prior knowledge about the oracle and its training data into three categories---\textit{domain knowledge}, \textit{deployment knowledge}, and \textit{model access}.

\subsubsection{\textbf{Domain Knowledge}}
Domain knowledge describes what the adversary knows about the task the model is designed for.
For example, if the model is an image classifier, then the model output should not change under standard image data augmentations, such as shifts, rotations, or crops. 
Usually, the adversary should be assumed to have as much domain knowledge as the oracle's designer.

In some domains, it is reasonable to assume the adversary has access to public task-relevant pretrained models or datasets. This is often the case for learning-based model extraction, which we develop in Section~\ref{ssec:imagenet}. We consider an adversary using part of a public dataset of 1.3 million images~\cite{deng2009imagenet} as unlabeled data to mount an attack against a model trained on a proprietary dataset of 1 billion labeled images~\cite{mahajan2018exploring}. 

\paragraph{Learning-based extraction is hard without natural data}
In learning-based extraction, we assume that the adversary is able to collect public \textit{unlabeled} data to mount their attack. This is a natural assumption for a theft-motivated adversary who wishes to steal the oracle for local use---the adversary has data they want to learn the labels of without querying the model! For other adversaries, progress in generative modeling is likely to offer ways to remove this assumption~\cite{micaelli2019zero}. We leave this to future work because our overarching aim in this paper is to characterize the model extraction attacker space around the notions of accuracy and fidelity. All progress achieved by our approaches is complementary to possible progress in synthetic data generation.

\subsubsection{\textbf{Deployment Knowledge}}
Deployment knowledge describes what the adversary knows about the oracle itself, including the model architecture, training procedure, and training dataset. The adversary may have access to public artifacts of the oracle---a distilled version of the oracle may be available (such as for OpenAI GPT~\cite{radford2019language}) or the oracle may be transfer learned from a public pretrained model (such as many image classifiers~\cite{sharif2014cnn} or language models like BERT~\cite{devlin2018bert}).

In addition, the adversary may not even know the features (the exact inputs to the model) or the labels (the classes the model may output). While the latter can generally be inferred by interacting with the model (e.g., making queries and observing the labels predicted by the model), inferring the former is usually more difficult. Our preliminary investigations suggest that these are not limiting assumptions, but we leave proper treatment of these constraints to future work.

\subsubsection{\textbf{Model Access}}
Model access describes the information the adversary obtains from the oracle, including bounds on how many queries the adversary may make as well as the oracle's response:
\begin{itemize}[itemsep=3pt,topsep=3pt,parsep=3pt,partopsep=3pt]
\item \textit{label}: only the label of the most-likely class is revealed.
\item \textit{label and score}: in addition to the most-likely label, the confidence score of the model in its prediction for this label is revealed.
\item \textit{top-$k$ scores}: the labels and confidence scores for the $k$ classes whose confidence are highest are revealed.
\item \textit{scores}: confidence scores for all labels are revealed.
\item \textit{logits}: raw logit values for all labels are revealed.
\end{itemize}
In general, the more access an adversary is given, the more effective they should be in accomplishing their goal. We instantiate practical attacks under several of these assumptions. Limiting model access has also been discussed as a defensive measure, as we elaborate in Section~\ref{sec:relwork}.

\section{Learning-based Model Extraction}
\label{sec:learning-based}

We present our first attack strategy where the victim model serves as a labeling oracle for the adversary. While many attack variants exist~\cite{tramer2016stealing,papernot2017practical}, they generally stage an iterative interaction between the adversary and the oracle, where the adversary collects labels for a set of points from the oracle and uses them as a training set for the extracted model. These algorithms are typically designed for accuracy extraction; in this section, we will demonstrate improved algorithms for accuracy extraction, using task-relevant unlabeled data.

\label{ssec:imagenet}

We realistically simulate large-scale model extraction by considering an oracle that was trained on 1 billion Instagram images~\cite{mahajan2018exploring} to obtain (at the time of the experiment) state-of-the-art performance on the standard image classification benchmark, ImageNet~\cite{deng2009imagenet}. The oracle, with 193 million parameters, obtained 84.2\% top-1 accuracy and 97.2\% top-5 accuracy on the 1000-class benchmark---we refer to the model as the "WSL model", abbreviating the paper title. We give the adversary access to the public ImageNet dataset. The adversary's goal is to use the WSL model as a labeling oracle to train an ImageNet classifier that performs better than if we trained the model directly on ImageNet. \textit{The attack is successful if access to the WSL model---trained on 1 billion proprietary images inaccessible to the adversary---enables the adversary to extract a model that outperforms a baseline model trained directly with ImageNet labels.} This is accuracy extraction for the ImageNet distribution, given unlabeled ImageNet training data.

We consider two variants of the attack: one where the adversary selects 10\% of the training set (i.e., about 130,000 points) and the other where the adversary keeps the entire training set (i.e., about 1.3 million points). To put this number in perspective, recall that each image has a dimension of 224x224 pixels and 3 color channels, giving us $224\cdot 224\cdot 3=150,528$ total input features. Each image belongs to one of 1,000 classes. Although ImageNet data is labeled, we always treat it as unlabeled to simulate a realistic adversary. 

\subsection{Fully-supervised model extraction}

The first attack is fully supervised, as proposed by prior work~\cite{tramer2016stealing}. It serves to compare our subsequent attacks to prior work, and to validate our hypothesis that labels from the oracle are more informative than dataset labels.

The adversary needs to obtain a label for each of the points it intends to train the extracted model with. Then it queries the oracle to label its training points with the oracle's predictions. The oracle reveals \textit{labels and scores} (in the threat model from Section~\ref{sec:taxonomy}) when queried. 

The adversary then trains its model to match these labels using the cross-entropy loss. We used a distillation temperature of $T=1.5$ in our experiments after a random search. Our experiments use two architectures known to perform well on image classification: ResNet-v2-50 and ResNet-v2-200. 

\bigskip
\noindent \textbf{Results.} We present results in Table~\ref{tab:imagenet_top5}. For instance, the adversary is able to improve the accuracy of their model by $1.0\%$ for ResNetv2-50 and $1.9\%$ for ResNet\_v2\_200 after having queried the oracle for 10\% of the ImageNet data. Recall that the task has 1,000 labels, making these improvements significant. The gains we are able to achieve as an adversary are in line with progress that has been made by the computer vision community on the ImageNet benchmark over recent years, where the research community improved the state-of-the-art top-1 accuracy by about one percent point per year.\footnote{https://paperswithcode.com/sota/image-classification-on-imagenet}

\begin{table*}[]
    \centering
\footnotesize{
    \begin{tabular}{cc|cccccc}

    \toprule
    Architecture & Data Fraction & ImageNet & WSL & WSL-5 & ImageNet + Rot & WSL + Rot & WSL-5 + Rot \\
    \midrule
    Resnet\_v2\_50 & 10\% & (81.86/82.95) & (82.71/84.18) & (82.97/84.52) & (82.27/84.14) & (82.76/84.73) & (82.84/84.59) \\
    
    Resnet\_v2\_200 & 10\% & (83.50/84.96) & (84.81/86.36) & (85.00/86.67) & (85.10/86.29) & (86.17/88.16) & (86.11/87.54) \\
    
    Resnet\_v2\_50 & 100\% & (92.45/93.93) & (93.00/94.64) & (93.12/94.87) & N/A & N/A & N/A\\
    
    Resnet\_v2\_200 & 100\% & (93.70/95.11) & (94.26/96.24) & (94.21/95.85) & N/A & N/A & N/A \\
    \bottomrule
    
    \end{tabular}
    }
    \caption{Extraction attack (top-5 accuracy/top-5 fidelity) of the WSL model~\cite{mahajan2018exploring}. Each row contains an architecture and fraction of public ImageNet data used by the adversary. ImageNet is a baseline using only ImageNet labels. WSL is an oracle returning WSL model probabilities. WSL-5 is an oracle returning only the top 5 probabilities. Columns with (+ Rot) use rotation loss on unlabeled data (rotation loss was not run when all data is labeled). An adversary able to query WSL always improves over ImageNet labels, even when given only top 5 probabilities. Rotation loss does not significantly improve the performance on ResNet\_v2\_50, but provides a (1.36/1.80) improvement for ResNet\_v2\_200, comparable to the performance boost given by WSL labels on 10\% data. In the high-data regime, where we observe a (0.56/1.13) improvement using WSL labels.}
    \label{tab:imagenet_top5}

\end{table*}

\begin{table}[]
    \centering
    \footnotesize{
    \begin{tabular}{c|c|ccc}
    \toprule
    Dataset & Algorithm & 250 Queries & 1000 Queries & 4000 Queries \\
    \midrule
    SVHN & FS & (79.25/79.48) & (89.47/89.87) & (94.25/94.71) \\
    SVHN & MM & (95.82/96.38) & (96.87/97.45) & (97.07/97.61) \\
    CIFAR10 & FS & (53.35/53.61) & (73.47/73.96) & (86.51/87.37) \\
    CIFAR10 & MM & (87.98/88.79) & (90.63/91.39) & (93.29/93.99) \\
    \bottomrule
    \end{tabular}
    }
    \caption{Performance (accuracy/fidelity) of fully supervised (FS) and MixMatch (MM) extraction on SVHN and CIFAR10. MixMatch with 4000 labels performs nearly as well as the oracle for both datasets, and MixMatch at 250 queries beats fully supervised training at 4000 queries for both datasets.}
    \label{tab:mixmatch_ssl}
\end{table}

\subsection{Unlabeled data improves query efficiency}

For adversaries interested in theft, a learning-based strategy should minimize the number of queries required to achieve a given level of accuracy. A natural approach towards this end is to take advantage of advances in label-efficient ML, including active learning~\cite{angluin1988queries} and semi-supervised learning~\cite{blum1998combining}. 

Active learning allows a learner to query the labels of arbitrary points---the goal is to query the best set of points to learn a model with. Semi-supervised learning considers a learner with some labeled data, but much more unlabeled data---the learner seeks to leverage the unlabeled data (for example, by training on guessed labels) to improve classification performance. Active and semi-supervised learning are complementary techniques~\cite{song2020combining, simeoni2020rethinking}; it is possible to pick the best subset of data to train on, while also using the rest of the unlabeled data without labels.

The connection between label-efficient learning and learning-based model extraction attacks is not new \cite{tramer2016stealing, DBLP:journals/corr/abs-1811-02054, DBLP:journals/corr/abs-1905-09165}, but has focused on active learning. \textit{We show that, assuming access to unlabeled task-specific data, semi-supervised learning can be used to improve model extraction attacks.} This could potentially be improved further by leveraging active learning, as in prior work, but our improvements are overall complementary to approaches considered in prior work. We explore two semi-supervised learning techniques: rotation loss~\cite{zhai2019s} and MixMatch~\cite{berthelot2019mixmatch}.

\bigskip
\noindent \textbf{Rotation loss.}
We leverage the current state-of-the-art semi-supervised learning approach on ImageNet, which augments the model with a \textit{rotation loss}~\cite{zhai2019s}. The model contains two linear classifiers from the second-to-last layer of the model: the classifier for the image classification task, and a rotation predictor. The goal of the rotation classifier is to predict the rotation applied to an input---each input is fed in four times per batch, rotated by $\lbrace0^{\circ}, 90^{\circ}, 180^{\circ}, 270^{\circ}\rbrace$. The classifier should output one-hot encodings $\lbrace OH(0; 4), OH(1; 4), OH(2; 4), OH(3; 4)\rbrace$, respectively, for these rotated images. Then, the rotation loss is written:
\[
L_R(X; f_{\theta}) = \frac{1}{4N}\sum_{i=0}^N \sum_{j=1}^{r} H(f_{\theta}(R_j(x_i)), j)
\]
where $R_j$ is the $j$th rotation, $H$ is cross-entropy loss, and $f_{\theta}$ is the model's probability outputs for the rotation task. Inputs need not be labeled, hence we compute this loss on unlabeled data for which the adversary did not query the model. That is, we train the model on both unlabeled data (with rotation loss), and labeled data (with standard classification loss), and both contribute towards learning a good representation for all of the data, including the unlabeled data.

We compare the accuracy of models trained with the rotation loss on data labeled by the oracle and data with ImageNet labels. Our best performing extracted model, with an accuracy of $64.5\%$, is trained with the rotation loss on oracle labels whereas the baseline on ImageNet labels only achieves $62.5\%$ accuracy with the rotation loss and $61.2\%$ without the rotation loss. This demonstrates the cumulative benefit of adding a rotation loss to the objective and training on oracle labels for a theft-motivated adversary. 

We expect that as semi-supervised learning techniques on ImageNet mature, further gains should be reflected in the performance of model extraction attacks. 

\bigskip
\noindent\textbf{MixMatch.} To validate this hypothesis, we turn to smaller datasets where semi-supervised learning has made significant progress. We investigate a technique called MixMatch~\cite{berthelot2019mixmatch} on two datasets: SVHN~\cite{netzer2011reading} and CIFAR10~\cite{krizhevsky2009learning}. MixMatch uses a combination of techniques, including training on "guessed" labels, regularization, and image augmentations.

For both datasets, inputs are color images of 32x32 pixels belonging to one of 10 classes. The training set of SVHN contains 73257 images and the test set contains 26032 images. The training set of CIFAR10 contains 50000 images and the test set contains 10000 images.
We train the oracle with a WideResNet-28-2 architecture on the labeled training set. The oracles achieve 97.36\% accuracy on SVHN and 95.75\% accuracy on CIFAR10.

The adversary is given access to the same training set but without knowledge of the labels. 
Our goal is to validate the effectiveness of semi-supervised learning by demonstrating that the adversary only needs to query the oracle on a small subset of these training points to extract a model whose accuracy on the task is comparable to the oracle's. To this end, we run 5 trials of fully supervised extraction (no use of unlabeled data), and 5 trials of MixMatch, reporting for each trial the median accuracy of the 20 latest checkpoints, as done in~\cite{berthelot2019mixmatch}.

\bigskip
\noindent\textbf{Results.}
In Table~\ref{tab:mixmatch_ssl}, we find that with only 250 queries (293x smaller label set than the SVHN oracle and 200x smaller for CIFAR10), MixMatch reaches 95.82\% test accuracy on SVHN and 87.98\% accuracy on CIFAR10. This is higher than fully supervised training that uses 4000 queries. With 4000 queries, MixMatch is within 0.29\% of the accuracy of the oracle on SVHN, and 2.46\% on CIFAR10. The variance of MixMatch is slightly higher than that of fully supervised training, but is much smaller than the performance gap. These gains come from the prior MixMatch is able to build using the unlabeled data, making it effective at exploiting few labels. We observe similar gains in test set fidelity.



\section{Limitations of Learning-Based Extraction}
\label{sec:fid_limits}

Learning-based approaches have several sources of non-determinism: the random initializations of the model parameters, the order in which data is assembled to form batches for SGD, and even non-determinism in GPU instructions~\cite{sculley2015hidden, lakshminarayanan2017simple}. Non-determinism impacts the model parameter values obtained from training.
Therefore, even an adversary with full access to the oracle's training data, hyperparameters, etc., would still need all of the learner's non-determinism to achieve the \textit{functionally equivalent} extraction goal described in Section~\ref{sec:taxonomy}. 
In this section, we will attempt to quantify this: for a strong adversary, with access to the exact details of the training setup, we will present an experiment to determine the limits of learning-based algorithms to achieving fidelity extraction.

We perform the following experiment.
We query an oracle to obtain a labeled substitute dataset $\mathcal{D}$.
We use $\mathcal{D}$ for a learning-based extraction attack which produces a model $f_\theta^1(x)$.
We run the learning-based attack a second time using $\mathcal{D}$, but with different sources of non-determinism to obtain a new set of parameters $f_\theta^2(x)$. 
If there are points $x$ such that 
$f_{\theta}^1(x)\neq f_{\theta}^2(x)$, then the prediction on $x$ is dependent not on the oracle, 
but on the non-determinism of the learning-based attack 
strategy---we are unable to guarantee fidelity.

We independently control the initialization
randomness and batch randomness during training on Fashion-MNIST~\cite{xiao2017/online} with
fully supervised SGD (we use Fashion-MNIST for training speed). 
We repeated each run 10 times and measure agreement between the ten obtained
models on the test set, adversarial examples generated by running FGSM with $\epsilon=25/255$ with the oracle model and the test set, and uniformly random inputs. The oracle uses initialization seed 0 and SGD seed 0---we also use two different initialization and SGD seeds.

Even when both training and initialization randomness are fixed (so that only GPU non-determinism remains), fidelity peaks at 93.7\% on the test set (see Table~\ref{tab:nondeterminism}).
With no randomness fixed, extraction achieves 93.4\% fidelity on the test set.
(Agreement on the test set should should be considered in reference to the base test accuracy of 90\%.)
Hence, even an adversary who has the victim model's \emph{exact} training set will be unable to exceed  \textasciitilde 93.4\% fidelity.
Using prototypicality metrics, as investigated in Carlini et al.~\cite{carlini2019prototypical}, we notice that test points where fidelity is easiest to achieve are also the most prototypical (i.e., more representative of the class it is labeled as). This connection is explored further in Appendix~\ref{app:prototypical}.
The experiment of this section is also related to uncertainty estimation using deep ensembles~\cite{lakshminarayanan2017simple}; we believe a deeper connection may exist between the fidelity of learning-based approaches and uncertainty estimation. Also relevant is the work mentioned earlier in Section~\ref{sec:taxonomy}, that shows that random networks are hard for learning-based approaches to extract. Here, we find that learning-based approaches have limits even for trained networks, on some portion of the input space.



\begin{table}[]
    \centering
    \footnotesize{
    \begin{tabular}{c|cccc}
    \toprule
    Query Set & Init \& SGD & Same SGD & Same Init & Different \\
    \midrule
    Test & 93.7\% & 93.2\% & 93.1\% & 93.4\% \\
    Adv Ex & 73.6\% & 65.4\% & 65.3\% & 67.1\% \\
    Uniform & 65.7\% & 60.2\% & 59.0\% & 60.2\% \\
    \bottomrule
    \end{tabular}
    }
    \caption{Impact of non-determinism on extraction fidelity. Even models extracted using the same SGD and initialization randomness as the oracle do not reach 100\% fidelity.}
    \label{tab:nondeterminism}
\end{table}

It follows from these arguments that non-determinism of both the victim and extracted model's learning procedures potentially compound, limiting the effectiveness of using a learning-based approach to reaching high fidelity. 

\input{dna_main}

\section{Hybrid Strategies}
\label{sec:hybrid}
Until now the strategies we have developed for extraction have been
pure and focused entirely on learning or entirely on direct extraction.
We now show that there is a continuous spectrum from which we can draw
attack strategies, and these \emph{hybrid} strategies can leverage both the
query efficiency of learning extraction, and the fidelity of direct extraction.

\subsection{Learning-Based Extraction with Gradient Matching}

Milli \emph{et al.} demonstrate that \emph{gradient matching}
helps extraction by optimizing the objective function
\[
\sum_{i=1}^n H(\oracle(x_i), f(x_i)) + \alpha|\nabla_x\oracle(x_i) - \nabla_xf(x_i)|_2^2,
\]
assuming the adversary can query the model for $\nabla_x\oracle(x)$. This is more model access than we permit our adversary, but is an example of using intuition from direct recovery to improve extraction. We found in preliminary experiments that this technique can improve fidelity on small datasets (increasing fidelity from 95\% to 96.5\% on Fashion-MNIST), but we leave scaling and removing the model access assumption of this technique to future work.
Next, we will show another combination of learning and direct recovery, using learning to alleviate some of the limitations of the previous functionally-equivalent extraction attack.

\subsection{Error Recovery through Learning}
Recall from earlier that the functionally-equivalent extraction attack fidelity
degrades as the model size increases. This is a result of low-probability errors
in the first weight matrix inducing incorrect biases on the first layer,
which in turn propagates and causes worse errors in the second layer.

We now introduce a method for performing a learning-based error recovery
routine. While performing a fully-learning-based attack leaves too many free
variables so that functionally-equivalent extraction is not possible, if we
fix many of the variables to the values extracted through the direct recovery
attack, we now show it is possible to learn the remainder of the variables.

Formally, let $\hat{A}^{(0)}$ be the extracted weight matrix for the first
layer and $\hat{B}^{(0)}$ be the extracted bias vector for the first layer.
Previously, we used least squares to directly solve for 
$\hat{A}^{(1)}$ and $\hat{B}^{(1)}$ assuming we had extracted the first layer
perfectly.
Here, we relax this assumption.
Instead, we perform gradient descent optimizing for parameters $W_{0..2}$ that
minimize
\[\mathbb{E}_{x \in \mathcal{D}} \big\lVert f_\theta(x) - W_1\relu(\hat{A}^{(0)} x + \hat{B}^{(0)} + W_0) + W_2   \big\rVert\]

That is, we use a single trainable parameter to adjust the bias term
of the first layer, and then solve (via gradient descent with training data) for the 
remaining weights accordingly.

This hybrid strategy increases the fidelity of the extracted model substantially,
detailed in Table~\ref{tab:fe_results2}. In the worst-performing example from earlier (with
only direct extraction) the extracted 128-neuron network had $80\%$ 
fidelity agreement with the victim model. When performing learning-based
recovery, the fidelity agreement jumps all the way to $99.75\%$.

\begin{table}
\begin{tabular}{c|rrrr}
\toprule
\textbf{\# of Parameters} & 50,000 & 100,000 & 200,000 & 400,000 \\
\midrule
\textbf{Fidelity} & 100\% & 100\% & 99.95\% & 99.31\% \\
\midrule
\textbf{Queries} & $2^{19.2}$ & $2^{20.2}$ & $2^{21.2}$ & $2^{22.2}$ \\
\bottomrule
\end{tabular}
\caption{Fidelity of extracted MNIST model is improved with the hybrid strategy. Note when comparing to Table~\ref{tab:fe_results} the model sizes are $4\times$ larger.}
\label{tab:fe_results2}
\end{table}

\subsubsection{Transferability}
\label{sec:transfer}
Adversarial examples \emph{transfer}: an adversarial example \cite{szegedy2013intriguing} generated on
one model often fools different models, too.
Transferability is higher when the models are more similar \cite{papernot2017practical}.

We should therefore expect that we can generate adversarial examples on our
extracted model, and that these will fool the remote oracle nearly always.
In order to measure transferability, we run 20 iterations of PGD \cite{madry2017towards}
with $\ell_{\infty}$ distortion set to the value most often used in the literature: for MNIST: $0.1$,
and for CIFAR-10: $0.03$.


The attack achieves functionally equivalent extraction (modulo floating point precision errors in the extracted weights), so we expect it to have high adversarial example transferability.
Indeed, we find we achieve a $100\%$ transferability success rate for all extracted models.

\begin{table}
\begin{tabular}{c|rrrr}
\toprule
\textbf{\# of Parameters} & 50,000 & 100,000 & 200,000 & 400,000 \\
\midrule
\textbf{Transferability} & 100\% & 100\% & 100\% & 100\% \\
\bottomrule
\end{tabular}
\caption{Transferability rate of adversarial examples using the extracted
neural network from our Section~\ref{sec:hybrid} attack.}
\label{tab:fe_results2}
\end{table}




\section{Related Work}
\label{sec:relwork}
Defenses for model extraction have fallen into two camps: limiting the information gained per query, and differentiating extraction adversaries from benign users. Approaches to limiting information include perturbing the probabilities returned by the model~\cite{tramer2016stealing, DBLP:journals/corr/abs-1811-02054,lee2018defending}, removing the probabilities for some of the model's classes~\cite{tramer2016stealing}, or returning only the class output~\cite{tramer2016stealing, DBLP:journals/corr/abs-1811-02054}. Another proposal has considered sampling from a distribution over model parameters~\cite{alabdulmohsin2014adding, DBLP:journals/corr/abs-1811-02054}. The other camp, differentiating benign from malicious users, has focused on analyzing query patterns~\cite{juuti2018prada,kesarwani2018model}. Non-adaptive attacks (such as supervised or MixMatch extraction) bypass query pattern-based detection, and are weakened by information limiting. We demonstrate the impact of removing complete access to probability values by considering only access to top 5 probabilities from WSL in Table~\ref{tab:imagenet_top5}. Our functionally-equivalent attack is broken by all of these measures. We leave consideration of defense-aware attacks to future work.

Queries to a model can also reveal hyperparameters~\cite{wang2018stealing} or architectural information~\cite{oh2017towards}. Adversaries can use side channel attacks to do the same~\cite{batina2018csi, hong2018security}. These are orthogonal to, but compatible with, our work---information about a model, such as assumptions made in Section~\ref{sec:dna}, empowers extraction.

Watermarking neural networks has been proposed~\cite{zhang2018protecting, uchida2017embedding} to identify extracted models.
Model extraction calls into question the utility of cryptographic protocols used to protect model weights.
One unrealized approach is obfuscation~\cite{barak2001possibility}, where an equivalent program could be released and queried as many times as desired.
A practical approach is secure multiparty computation, where each query is computed by running a protocol between the model owner and querier~\cite{barni2006privacy}.

\section{Conclusion}

This paper characterizes and explores the space of model extraction attacks
on neural networks.
We focus this paper specifically around the objectives of \emph{accuracy}, to measure the success of a theft-motivated adversary, and
\emph{fidelity}, an often-overlooked measure which compares the agreement between models to reflect the success of a recon-motivated adversary.

Our learning-based methods can effectively attack a model with several millions of
parameters trained on a billion images, and allows the attacker to reduce the
error rate of their model by 10\%.
This attack does not match perfect fidelity with the victim model due to
what we show are inherent limitations of learning-based approaches: nondeterminism (including only the nondeterminism on the GPU) prohibits training identical models.
In contrast, our direct functionally-equivalent extraction returns
a neural network agreeing with the victim model on $100\%$ of the test
samples and having $100\%$ fidelity on transfered adversarial examples.

We then propose a hybrid method which unifies these two attacks, using learning-based approaches to recover from numerical
instability errors when performing the functionally-equivalent extraction attack.

Our work highlights many remaining open problems in model extraction, such as reducing the capabilities required by our attacks and scaling functionally-equivalent extraction.

\section*{Acknowledgements}
We would like to thank Ilya Mironov for lengthy and fruitful discussions regarding the functionally equivalent extraction attack. We also thank \'{U}lfar Erlingsson for helpful discussions on positioning the work, and Florian Tram\`{e}r for his comments on an early draft of this paper.

\bibliographystyle{IEEEtran}
\bibliography{references}


\appendices

\input{hardness_formal}

\input{imagenet_top1}
\input{prototypes}

\input{dna_supp}

\clearpage

\end{document}

%% file: dna_main.tex
\section{Functionally Equivalent Extraction}
\label{sec:dna}

Having identified fundamental limitations that prevent learning-based approaches 
from perfectly matching the oracle's mistakes, we now turn to a different approach where
the adversary extracts the oracle's weights directly, seeking to achieve functionally-equivalent extraction.

This attack can be seen as an extension of two prior works.
\begin{itemize}
    \item Milli \emph{et al.}~\cite{milli2018model} introduce an attack to extract
    neural network weights under the assumption that
    the adversary is able to make \emph{gradient queries}.
    That is, each query the adversary makes reveals not only the prediction of the
    neural network, but also the gradient of the neural network with respect to the query.
    To the best of our knowledge this is the only functionally-equivalent extraction
    attack on neural networks with one hidden layer, although it was not actually
    implemented in practice.
    \item Batina \emph{et al.}~\cite{batina2018csi}, at USENIX Security 2019,
    develop a side-channel attack that
    extracts neural network weights through monitoring the power use of a
    microprocessor evaluating the neural network. This is a much more powerful
    threat model than made by any of the other model extraction papers.
    To the best of our knowledge this is the only practical direct model
    extraction result---they manage to extract essentially arbitrary depth networks.
\end{itemize}

In this section we introduce an attack which only requires standard
queries  (i.e., that return the model's prediction instead of its gradients) and does not require any side-channel leakages, yet still manages to achieve \emph{higher fidelity} extraction than the side-channel extraction work for two-layer networks, assuming double-precision inference.

\bigskip
\noindent
\paragraph{Attack Algorithm Intuition.}
As in \cite{milli2018model}, our
attack is tailored to work on neural networks with the ReLU activation
function (the ReLU is an effective default choice of activation function~\cite{nair2010rectified}).
This makes the neural network a piecewise linear function.
Two samples are within the same linear region if all ReLU units have the
same sign, illustrated in Figure~\ref{fig:func_equiv}.

By finding adjacent linear regions, and computing the difference between them, we force a single ReLU to change signs. Doing this, it is possible to almost completely
determine the weight vector going into that ReLU unit.
Repeating this attack for all ReLU units lets us recover
the first weight matrix completely.
(We say almost here, because we must do some work to recover the sign of
the weight vector.)
Once the first layer of the two-layer neural network has been determined,
the second layer can be uniquely solved for algebraically through least squares.
This attack is optimal up to a constant factor---the query complexity is discussed in Appendix~\ref{app:dna_queries}.

\subsection{Notation and Assumptions}
As in \cite{milli2018model}, we only aim to extract neural networks
with one hidden layer using the ReLU activation function.
We denote the model weights by $A^{(0)}\in\mathbb{R}^{d\times h}, A^{(1)}\in\mathbb{R}^{h\times K}$ and biases by $B^{(0)}\in\mathbb{R}^{h}, B^{(1)}\in\mathbb{R}^{K}$. Here, $d, h$, and $K$ respectively refer to the input dimensionality, the size of the hidden layer, and the number of classes. This is found in Table~\ref{tab:diffextrnotation}.

\begin{table}
\label{tab:diffextrnotation}
\centering
\begin{tabular}{cc}
\toprule
    Symbol & Definition \\
    \midrule
    $d$ & Input dimensionality \\
    
    $h$ & Hidden layer dimensionality ($h < d$) \\
    
    $K$ & Number of classes \\
    
    $A^{(0)}\in\mathbb{R}^{d\times h}$ & Input layer weights \\
    
    $B^{(0)}\in\mathbb{R}^h$ & Input layer bias \\
    
    $A^{(1)}\in\mathbb{R}^{h\times K}$ & Logit layer weights \\
    
    $B^{(1)}\in\mathbb{R}^K$ & Logit layer bias \\
    \bottomrule
\end{tabular}
\vspace{.5em}
\caption{Parameters for the functionally-equivalent attack.}
\end{table}

We say that $\relu(x)$ is at a critical point if $x=0$; this is the
location at which the unit's gradient changes from $0$ to $1$.
We assume the adversary is able to observe the raw logit outputs as
64-bit floating point values. We will use the notation $\oracle_L$ to denote the logit oracle.
Our attack implicitly assumes that the rows of $A^{(0)}$ are linearly
independent. Because the dimension of the input space is larger
than the hidden space by at least 100, it is exceedingly unlikely
for the rows to be linearly dependent (and we find this holds true
in practice).

Note that our attack is not an SQ algorithm, which would only allow us to look at aggregate statistics of our dataset. Instead, our algorithm is very particular in its analysis of the network, computing the differences between linear regions, for example, cannot be done with aggregate statistics.
This structure allows us to avoid the pathologies of Section~\ref{sssc:func_hard}.

\begin{figure}
    \centering
    \includegraphics[width=\linewidth]{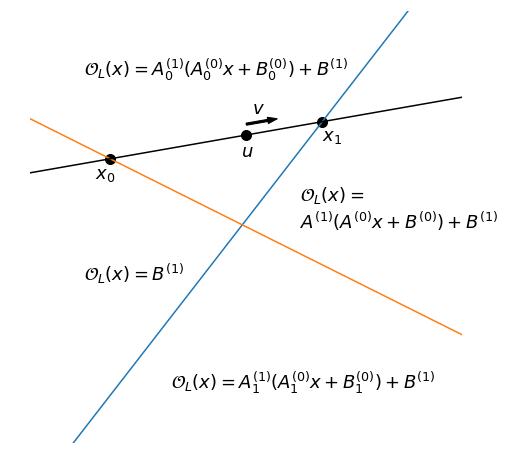}
    \caption{2-dimension intuition for the functionally equivalent extraction attack.}
    \label{fig:func_equiv}
\end{figure}

\subsection{Attack Overview}
The algorithm is broken into four phases:
\begin{itemize}
    \item \textbf{Critical point search} identifies inputs $\{x_i\}_{i=1}^n$ to the neural
    network so that exactly one of the ReLU units is at a critical point (i.e.,
    has input identically $0$).
    \item \textbf{Weight recovery} takes an input $x$ which causes
    the $i$th neuron to be at a critical point.
    We use this point $x$ to compute the difference between the
    two adjacent linear regions induced by the critical point,
    and thus the weight vector row $A^{(0)}_i$.
    By repeating this process for each ReLU we obtain the complete
    matrix $A^{(0)}$.
    Due to technical reasons discussed below, we can only recover the row-vector
    up to sign.
    \item \textbf{Sign recovery} determines the sign of each
    row-vector $A^{(0)}_j$ for all $j$ using global information about $A^{(0)}$.
    \item \textbf{Final layer extraction} uses algebraic techniques (least
    squares) to solve for the second layer of the network.
    
\end{itemize}

\subsection{Critical Point Search}
\label{ssec:critpointsearch}
For a two layer network, observe that 
the logit function is given by the equation
$\oracle_L(x) = A^{(1)}\relu(A^{(0)}x+B^{(0)}) + B^{(1)}$. 
To find a critical point for every ReLU, 
we sample two random vectors $u,v\in\mathbb{R}^d$, and consider the function
\[
L(t; u, v, \oracle_L) = \oracle_L(u+tv).
\]
for $t$ varying between a small and large appropriately selected value (discussed below).
This amounts to drawing a line in the inputs of the network; passed through ReLUs, this line becomes the piecewise linear function $L(\cdot)$. The points $t$ where $L(t)$ is non-differentiable are exactly locations where some $\relu_i$ is changing signs
(i.e., some ReLU is at a critical point). 
Figure~\ref{fig:crit_search} shows an example of what this sweep looks like on a trained MNIST model.

\begin{figure}
    \centering
    \includegraphics[width=\linewidth]{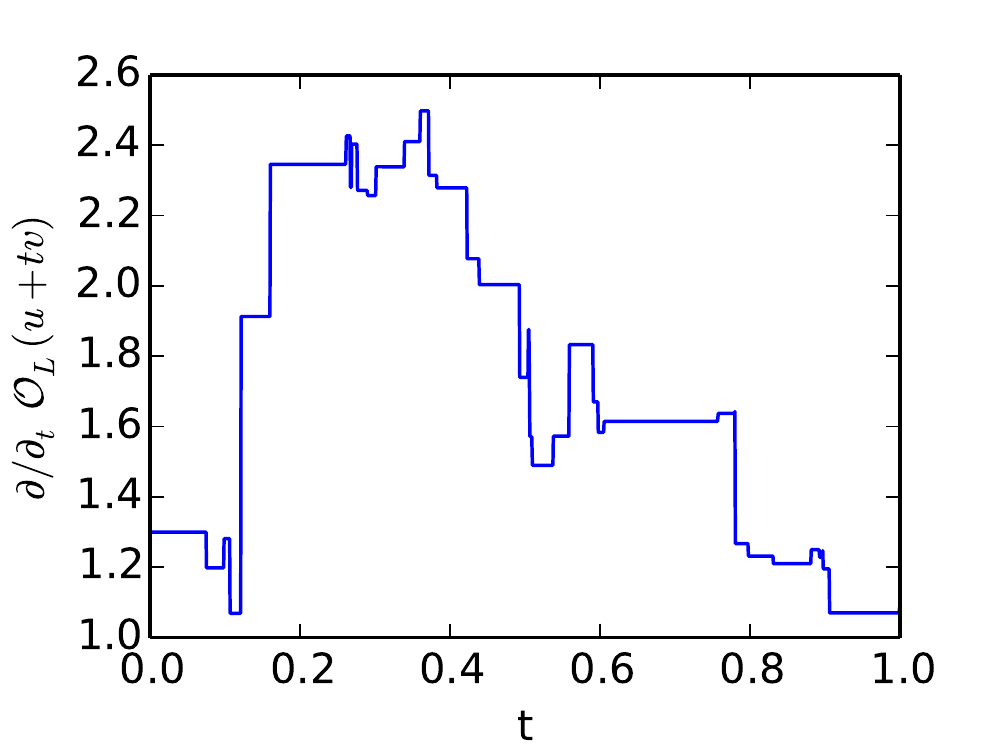}
    \caption{An example sweep for critical point search. Here we plot the partial derivative across $t$ and see that $\oracle_L(u+tv)$ is piecewise linear, enabling a binary search.}
    \label{fig:crit_search}
\end{figure}

Furthermore, notice that given a pair $u, v$, there is exactly one value $t$ for which each ReLU is at a critical point, 
and if $t$ is allowed to grow
arbitrarily large or small that every ReLU unit will switch sign exactly once.
Intuitively, the reason this is true is that each ReLU's input, (say $wx + b$ for some $w, b$), is a monotone function of $t$ ($w^Tu t+w^Tv + b$).
Thus, by varying $t$, we can identify an input $x_i$ that sets the $i$th ReLU to 0
for every relu $i$ in the network.
This assumes we are not moving parallel to any of the rows (where $w^Tu=0$), and that we vary $t$ within a sufficiently large interval (so the $w^Tut$ term may overpower the constant term). The analysis of \cite{milli2018model} suggests that these concerns can be resolved with high probability by varying $t\in\left[ -h^2, h^2 \right]$.

While in theory it would be possible to sweep all values of $t$ to identify the critical
points, this would require a large number of queries.
Thus, to efficiently search for the locations of critical points,
we introduce a refined search algorithm which improves on the
binary search as used in \cite{milli2018model}.
Standard binary search requires $O(n)$ model queries to obtain $n$ bits
of precision.
Therefore, we propose a refined technique which does not have this restriction
and requires just $O(1)$ queries to obtain high (20+ bits) precision.
The key observation we make is that if we are searching between two values $[t_1, t_2]$ and there is exactly one discontinuity in this range, we can precisely identify the location of that discontinuity efficiently.

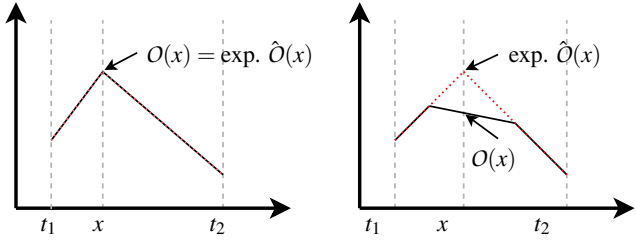
\begin{figure}
  \centering

  \scalebox{.9}{
  \begin{overpic}[unit=1mm]{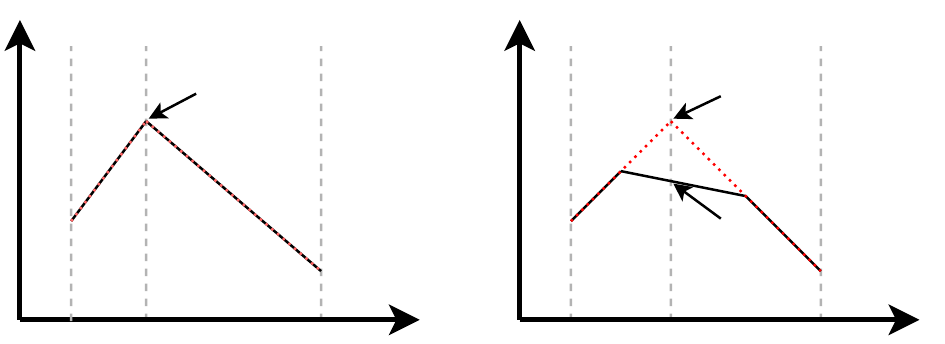}
    \put(6,-1){$t_1$}
    \put(14,-1){$x$}
    \put(32,-1){$t_2$}

    \put(56,-1){$t_1$}
    \put(67,-1){$x$}
    \put(82,-1){$t_2$}

    \put(22,25){$\oracle(x) = \text{exp. } \hat{\oracle}(x)$}

    \put(78,25){$\text{exp. } \hat{\oracle}(x)$}
    \put(72,9){$\oracle(x)$}
  \end{overpic}
  }

  \caption{Efficient and accurate 2-linear testing subroutine in Algorithm~\ref{alg:2-linear?}. Left shows a successful case where the algorithm
  succeeds; right shows a potential failure case, where there are
   multiple nonlinearities. We detect this by observing
  the expected value of $\oracle(x)$ is not the observed (queried)
  value.}
  \label{fig:2-linear?}
\end{figure}

\begin{algorithm}
\begin{algorithmic}
\State Function $f$, range $[t_1, t_2]$, $\epsilon$
\State $m_1 = \frac{f(t_1+\epsilon) - f(t_1)}{\epsilon}$ \Comment{Gradient at $t_1$}
\State $m_2 = \frac{f(t_2) - f(t_2-\epsilon)}{\epsilon}$ \Comment{Gradient at $t_2$}
\State $y_1 = f(a), y_2=f(b)$
\State $x = a + \frac{y_2-y_1-(b-a)m_2}{m_1-m_2}$ \Comment{Candidate critical point}
\State $\hat{y} = y_1+m_1\frac{y_2-y_1-(b-a)m_2}{m_1-m_2}$ \Comment{Expected value at candidate}
\State $y = f(x)$ \Comment{True value at candidate}
\If{$\hat{y} = y$} \Return $x$
\Else ~~\Return "More than one critical point"
\EndIf
\end{algorithmic}
\caption{Algorithm for 2-linearity testing. Computes the location of the only critical point in a given range or rejects if there is more than one.}
\label{alg:2-linear?}
\end{algorithm}

An intuitive diagram for this algorithm can be found in Figure~\ref{fig:2-linear?} and the algorithm can be found in Algorithm~\ref{alg:2-linear?}.
The property this leverages is that the function is piecewise linear--if we know the range is composed of two linear segments, we can identify the linear segments and compute their intersection.
In Algorithm~\ref{alg:2-linear?}, lines 1-3 describe computing the two linear regions' slopes and intercepts. Lines 4 and 5 compute the intersection of the two lines (also shown in the red dotted line of Figure~\ref{fig:2-linear?}). The remainder of the algorithm performs the correctness check, also illustrated in Figure~\ref{fig:2-linear?}; if there are more than 2 linear components, it is unlikely that the true function value will match the function value computed in line 5, and we can detect that the algorithm has failed.

\subsection{Weight Recovery}
\label{ssec:weightrecovery}
After running critical point search we obtain a set $\lbrace x_i\rbrace_{i=1}^h$, where each critical point corresponds to a point where a single ReLU flips sign.
In order to use this information to learn the weight matrix $A^{(0)}$ 
we measure the second derivative of $\oracle_L$ in each input direction at
the points $x_i$. Taking the second
derivative here corresponds to measuring the difference between the linear regions on either side of the ReLU.
Recall that prior work assumed direct access to gradient queries, and
thus did not require any of the analysis in this section.

\subsubsection{Absolute Value Recovery}
To formalize the intuition of comparing adjacent hyperplanes, observe that 
for the oracle $\oracle_L$ and for a critical point $x_i$ (corresponding to 
$\relu_i$ being zero) and for a random input-space direction $e_j$ we have
\[
\begin{split}
\left.\frac{\partial^2\oracle_L}{\partial e_j^2}\right|_{x_i} & = \left.\frac{\partial\oracle_L}{\partial e_j}\right|_{x_i+c\cdot e_j} - \left.\frac{\partial\oracle_L}{\partial e_j}\right|_{x_i - c\cdot e_j} \\
 & = \sum_k A^{(1)}_k\mathbbm{1}(A^{(0)}_k(x_i+c\cdot e_j)+B^{(0)}_k>0)A^{(0)}_{kj}\\
 & ~~ - \sum_k A^{(1)}_k\mathbbm{1}(A^{(0)}_k(x_i-c\cdot e_j)+B^{(0)}_k>0)A^{(0)}_{kj}\\
 & = A^{(1)}_i\left(\mathbbm{1}(A^{(0)}_i\cdot e_j > 0)-\mathbbm{1}(-A^{(0)}_i\cdot e_j > 0)\right)A^{(0)}_{ji}\\
 & = \pm(A^{(0)}_{ji} A^{(1)}_i)
\end{split}
\]

for a $c>0$ small enough so that $x_i \pm c\cdot e_j$ does not flip any other ReLU. Because $x_i$ is a critical point and $c$ is small, the sums in the second line differ only in the contribution of $\relu_i$.
However at this point we only have a product involving both weight matrices.
We now show this information is useful.

If we compute $|A^{(0)}_{1i} A^{(1)}|$ and $|A^{(0)}_{2i} A^{(1)}|$ by querying
along directions $e_1$ and $e_2$,
we can divide these quantities to obtain the value $|A^{(0)}_{1i}/A^{(0)}_{2i}|$,
the ratio of the two weights.
By repeating the above process for each input direction we can, for all $k$, obtain the pairwise
ratios $|A^{(0)}_{1i}/A^{(0)}_{ki}|$. 

Recall from Section~\ref{sec:taxonomy} that obtaining the ratios of weights is the
theoretically optimal result
we could hope to achieve. It is always possible to multiply all of the
weights \emph{into} a ReLU by a constant $c>0$ and then multiply
all of the weights \emph{out} of the $\relu$ by $c^{-1}$.
Thus, without loss of generality,
we can assign $A^{(0)}_{1i} = 1$ and scale the remaining entries accordingly.
Unfortunately, we have  lost a small amount of information here. 
We have only learned the absolute value of the ratio, and not the value itself.

\subsubsection{Weight Sign Recovery}
Once we reconstruct the values $|A^{(0)}_{ji}/A^{(0)}_{1i}|$ for all $j$
we need to recover the sign of these values.
To do this we consider the following quantity:
\[
\left.\frac{\partial^2\oracle_L}{\partial (e_{j}+e_{k})^2}\right|_{x_i} = \pm(A^{(0)}_{ji} A^{(1)}_i \pm A^{(0)}_{ki} A^{(1)}_i).
\]
That is, we consider what would happen if we take the second partial derivative
in the direction $(e_j + e_k)$.
Their contributions to the gradient will either cancel out, indicating $A^{0)}_{ji}$ and $A^{(0)}_{ki}$ are of opposite sign, or they will compound on each other, indicating they have the same sign. 
Thus, to recover signs, we can perform this comparison along each
direction $(e_1 + e_j)$.

Here we encounter one final difficulty. There are a total of $n$ signs we need
to recover, but because we compute the signs by comparing
ratios along different directions, we can only obtain $n-1$ relations.
That is, we now know the correct signed value of $A^{(0)}_i$ up to
a single sign for the entire row.

It turns out this is to be expected. What we have computed is the
normal direction to the hyperplane, but because any given
hyperplane can be described by an infinite number of normal vectors differing
by a constant scalar, we can not hope to use local information to recover this
final sign bit.

Put differently, while it is possible to push a constant $c>0$ through from the first layer
to the second layer, it is not possible to do this for negative constants,
because the ReLU function is not symmetric.
Therefore, it is necessary to learn the sign of this row.

\subsection{Global Sign Recovery}
Once we have recovered the input vector's weights, we still don't know the sign for the given inputs---we only measure the difference between linear functions at each critical point, but do not know which side is the positive side of the ReLU~\cite{milli2018model}. Now, we need to leverage global information in order to reconcile all of inputs' signs.

Notice that recovering $\hat{A}^{(0)}_i$ allows us to obtain $B^{(0)}_i$ by using the fact that $A^{(0)}_i\cdot x_i + B^{(0)}_i = 0$. Then we can compute $\hat{B}^{(0)}_i$ up to the same global sign as is applied to $\hat{A}^{(0)}_i$.

Now, to begin recovering sign, we search for a vector $z$ that is in the null space of $\hat{A}^{(0)}$, that is,
$\hat{A}^{(0)} z = \vec{0}$. Because the neural network has $h < d$, the null-space
is non-zero, and we can find many such vectors using least squares.
Then, for each $\relu_i$, we search for a vector $v_i$
such that $v_i A^{(0)} = e_i$ where here $e_i$ is the $i$th basis vector in the
hidden space.
That is, moving along the $v_i$ direction only changes $\relu_i$'s input value. 
Again we can search for this through least squares.

Given $z$ and these $v_i$ we query the neural network for the values of $\oracle_L(z)$, $\oracle_L(z+v_i)$, and $\oracle_L(z - v_i)$.
On each of these three queries, all hidden units are $0$ except for $\relu_i$ which
recieves as input either $0$, $1$, or $-1$
by the construction of $v_i$.
However, notice that the output of $\relu_i$ can only be either $0$ or $1$,
and the two $\{-1,0\}$ cases collapse to just output $0$.
Therefore, if $\oracle_L(z+v_i) = \oracle_L(z)$, we know that $A^{(0)}_i\cdot v_i < 0$. Otherwise, we will find $\oracle_L(z-v_i) = \oracle_L(z)$ and $A^{(0)}_i\cdot v_i > 0$. This allows us to recover the sign bit for $\relu_i$.

\subsection{Last Layer Extraction}
Given the completely extracted first layer, the logit function of the network is just a linear transformation which we can recover with least squares, through making $h$ queries where each ReLU is active at least once. In practice, we use the critical points discovered in the previous
section so that we do not need to make additional neural network queries.

\subsection{Results}

\noindent
\textbf{Setup.} We train several one-layer fully-connected neural networks with between
16 and 512 hidden units (for 12,000 and 100,000 trainable
parameters, respectively) on the MNIST~\cite{lecun1998gradient} and CIFAR-10 datasets~\cite{krizhevsky2009learning}.
We train the models with the Adam~\cite{kingma2014adam} optimizer for 20 epochs at batch
size 128 until they converge.
We train five networks of each size to obtain higher statistical significance. Accuracies of these networks can be found in the supplement in Appendix~\ref{app:dna_supp}. In Section~\ref{sec:learning-based}, we used 140,000$\approx 2^{17}$ queries for ImageNet model extraction. This is comparable to the number of queries used to extract the smallest MNIST model in this section, highlighting the advantages of both approaches.

\smallskip \noindent
\textbf{MNIST Extraction.}
We implement the functionally-equivalent extraction attack in JAX~\cite{JAX}
and run it on each trained oracle.
We measure the fidelity of the extracted model, comparing
predicted labels, on the MNIST test set.

Results are summarized in Table~\ref{tab:fe_results}.
For smaller networks, we achieve 100\% fidelity on the test set: every single one of
the $10,000$ test examples is predicted the same.
As the network size increases, low-probability errors we encounter become more common, but the extracted neural network still disagrees with the oracle on only $2$ of the $10,000$ examples.

Inspecting the weight matrix
that we extract and comparing it to the weight
matrix of the oracle classifier, we find that
we manage to reconstruct the first weight matrix to an average precision of 23 bits---we provide more results in Appendix~\ref{app:dna_supp}.

\smallskip \noindent
\textbf{CIFAR-10 Extraction.} 
Because this attack is data-independent, the underlying task
is unimportant for how well the attack works; only the number
of parameters matter.
The results for CIFAR-10 are thus identical to MNIST when
controlling for model size:
we achieve 100\% test set agreement on models with fewer than
$200,000$ parameters and  and greater than 99\% test set
agreement on larger models.

\smallskip \noindent
\textbf{Comparison to Prior Work.}
To the best of our knowledge, this is by orders of magnitude the highest fidelity
extraction of neural network weights. 

The only fully-implemented neural network
extraction attack we are aware of is the work of Batina \emph{et al.}~\cite{batina2018csi},
who uses an electromagnetic side channels and differential power analysis
to recover an MNIST neural network with neural network weights with an average error of 0.0025. In comparison, we are able to achieve an average error in the first
weight matrix for a similarly sized neural network of just 0.0000009---over
two thousand times more precise.
To the best of our knowledge no functionally-equivalent
CIFAR-10 models have
been extracted in the past.

We are unable to make a comparison between the fidelity of our extraction attack
and the fidelity of the attack presented in Batina \emph{et al.} because they
do not report on this number: they only report the accuracy of the extracted
model and show it is similar to the original model. We believe this strengthens
our observation that comparing across accuracy and fidelity is not currently widely
accepted as best practice.

\smallskip \noindent
\textbf{Investigating Errors.} We observe that as the number of parameters that
must be extracted increases, the fidelity of the model decreases.
We investigate why this happens and discovered that a small fraction
of the time (roughly 1 in 10,000) the gradient estimation procedure
obtains an incorrect estimate of the gradient and therefore one of
the extracted weights $\hat{A}^{(0)}_{ij}$ is incorrect by a non-insignificant
margin.

Introducing an error into just one of the weights of the
first matrix $\hat{A}^{(0)}$ should not induce significant further errors.
However, because of this error, when we solve for the bias
vector, the extracted
bias $\hat{B}^{(0)}_i$ will have error proportional to the error of 
$\hat{A}^{(0)}_{ij}$.
And when the bias is wrong, it impacts \emph{every} calculation, even those
where this edge is not in use.

Resolving this issue completely either requires reducing the failure
rate of gradient estimation from 1 in 10,000 to practically 0, or
would require a complex error-recovery procedure.
Instead, we will introduce in the following section an improvement
which almost completely solves this issue.

\smallskip \noindent
\textbf{Difficulties Extending the Attack.}
The attack is specific to two layer neural networks; deeper networks pose multiple difficulties.
In deep networks, the critical point search step of Section~\ref{ssec:critpointsearch} will result in critical points from many different layers, and determining which layer a critical point is on is nontrivial. 
Without knowing which layer a critical point is on, we cannot control inputs to the neuron, which we need to do to recover the weights in Section~\ref{ssec:weightrecovery}.
Even given knowledge of what layer a critical point is on, the inputs of any neuron past layer 1 are the outputs of other neurons, so we only have indirect control over their inputs.
Finally, even with the ability to recover these weights, small numerical errors occur in the first layer extraction.
These cause errors in every finite differences computation in further layers, causing the second layer to have even larger numerical errors than the first (and so on).
Therefore, extending the attack to deeper networks will require at least solving each of the following: producing critical points belonging to a specific layer, recovering weights for those neurons without direct control of their inputs, and significantly reducing numerical errors in these algorithms.

\begin{table}
\begin{tabular}{c|rrrr}
\toprule
\textbf{\# of Parameters} & 12,500 & 25,000 & 50,000 & 100,000 \\
\midrule
\textbf{Fidelity} & 100\% & 100\% & 100\% & 99.98\% \\
\midrule
\textbf{Queries} & $2^{17.2}$ & $2^{18.2}$ & $2^{19.2}$ & $2^{20.2}$ \\
\bottomrule
\end{tabular}
\caption{Fidelity of the functionally-equivalent extraction
attack across different test distributions on an MNIST victim model. Results are averaged over five extraction
attacks. For small models, we achieve perfect
fidelity extraction; larger models have near-perfect fidelity on the test data distribution, but begins to lose accuracy at $100,000$ parameters.}
\label{tab:fe_results}
\end{table}

%% file: hardness_formal.tex
\section{Formal Statements for Section~\ref{ssec:extracthard}}
\label{app:hardness_formal}
Here, we give the formal arguments for the difficulty of model extraction to support informal statements from Section~\ref{ssec:extracthard}.

\begin{reptheorem}{thm:rectangle}
There exists a class of width $3k$ and depth 2 neural networks on domain $[0, 1]^d$ (with precision $p$ numbers) with $d\ge k$ that require, given logit access to the networks, $\Theta(p^k)$ queries to extract.
\end{reptheorem}

In order to prove Theorem~\ref{thm:rectangle}, we introduce a family of functions we call \emph{$k$-rectangle bounded functions}, which we will show satisfies this property.
\begin{definition}
A function $f$ on domain $[0, 1]^d$ with range $\mathbb{R}$ is a \emph{rectangle bounded function} if there exists two vectors $a, b$ such that $f(x)\neq 0 \implies a\preceq x \preceq b$, where $\preceq$ denotes element-wise comparison. The function $f$ is a \emph{$k$-rectangle bounded function} if there are $k$ indices $i$ such that $a_i\neq 0$ or $b_i\neq 1$.
\end{definition}
Intuitively, a $k$-rectangle function only outputs a non-zero value on a multidimensional rectangle that is constrained in only $k$ coordinates. We begin by showing that we can implement $k$-rectangle functions for any $a, b$ using a ReLU network of width $k$ and depth 2.

\begin{lemma}
For any $a, b$ with $k$ indices $i$ such that $a_i\neq 0$ or $b_i\neq 1$, we can construct a $k$-rectangle bounded function for $a, b$ with a ReLU network of width $3k$ and depth 2.
\end{lemma}
\begin{proof}
We will start by constructing a 3-ReLU gadget with output $\ge 1$ only when $a_i\le x_i\le b_i$. We will then show how to compose $k$ of these gadgets, one for each index of the $k$-rectangle, to construct the $k$-rectangle bounded function.

The 3-ReLU gadget only depends on $x_i$, so weights for all other ReLUs will be set to 0. Observe that the function $T_i(x; a,b) = \relu(x-a) +\relu(x_i-b_i) - 2\relu(x_i-(a_i+b_i)/2)$ is nonzero only on the interval $(a_i,b_i)$. This is easier to see when it is written as
\begin{multline*}
\relu(x_i-a_i) - \relu(x_i-(a_i+b_i)/2) \\- (\relu(x_i-(a_i+b_i)/2) - \relu(x_i-b_i)).
\end{multline*}
The function $\relu(x-x_1) - \relu(x-x_2)$ with $x_1<x_2$ looks like a sigmoid, and has the following form:
\[
\relu(x-x_1) - \relu(x-x_2) = \begin{cases} 
      0 & x\leq x_1 \\
      x-x_1 & x_1\leq x\leq x_2 \\
      x_2-x_1 & x\geq x 
   \end{cases}
\]

Now, $T_i(x; a_i, b_i)\cdot 1/(b_i - a_i)$ has range $[0, 1]$ for any value of $a_i, b_i$. Then the function 
\[
f_{a,b}(x) = \relu( \sum_i(T_i(x; a_i, b_i) / (b_i - a_i)) - (k-1) )
\]
is $k$-rectangle bounded for vectors $a,b$. To see why, we need that no input $x$ not satisfying $a\preceq x \preceq b$ has $\sum_i(T_i(x; a_i, b_i) / (b_i - a_i)) > k-1$. This is simply because each term $T_i(x; a_i, b_i)\le 1$, so unless all $k$ such terms are $>0$, the inequality cannot hold.
\end{proof}

Now that we know how to construct a $k$-rectangle bounded function, we will introduce a set of $p^k$ disjoint $k$-rectangle bounded functions, and then show that any one requires $p^k$ queries to extract when the others are also possible functions.

\begin{lemma}
There exists a family of $k$-rectangle bounded functions $\mathcal{F}$ such that extracting an element of $\mathcal{F}$ requires $p^k$ queries in the worst case.
\end{lemma}
Here, $p$ is the feature precision; images with 8-bit pixels have $p=256$.
\begin{proof}
We begin by constructing $\mathcal{F}$. The following $p$ ranges are clearly pairwise disjoint: $\lbrace(\frac{i-1}{p}, \frac{i}{p})\rbrace_{i=1}^{p}$. Then pick any $k$ indices, and we can construct $p^k$ distinct $k$-rectangle bounded functions - one for each element in the Cartesian product of each index's set of ranges. Call this set $\mathcal{F}$.

The set of inputs with non-zero output is distinct for each function, because their rectangles are distinct. Now consider the information gained from any query. If the query returns a non-zero value, the function is learned. If not, at most one function from $\mathcal{F}$ is ruled out - the function whose rectangle was queried. Then any sequence of $n$ queries to an oracle can rule out at most $n$ of the functions of $\mathcal{F}$, so that at least $|\mathcal{F}|=p^k$ queries are required in the worst case. 
\end{proof}

Putting Lemma 1 and 2 together gives us Theorem~\ref{thm:rectangle}.

\begin{reptheorem}{thm:subsetsum}
Checking whether two networks with domains $\lbrace 0, 1\rbrace^d$ are functionally equivalent is NP-hard.
\end{reptheorem}
\begin{proof}
We prove this by reduction to subset sum. A similar reduction (reducing to 3-SAT instead of Subset Sum) for a different statement appears in \cite{katz2017reluplex}.

Suppose we receive a subset sum instance $T, p, [v_1, v_2, \cdots, v_d]$ - the set is $v$, the target sum is $T$, and the problem's precision is $p$. We will construct networks $f_1$ and $f_2$ such that checking if $f_1$ and $f_2$ are functionally equivalent is equivalent to solving the subset sum instance. We start by setting $f_1=0$ - it never returns a non-zero value. We now construct a network $f_2$ that has nonzero output only if the subset sum instance can be solved (and finding an input with nonzero output reveals the satisfying subset).

The network $f_2$ has three hidden units in the first layer with incoming weight for the $i$th feature equal to $v_i$. This means the dot product of the input $x$ with weights will be the sum of the subset $\lbrace i| x_i = 1\rbrace$. We want to force this to accept iff there is an input where this sum is $T$. To do so, we use the same 3-ReLU gadget as in the proof of Theorem~\ref{thm:rectangle}:
\begin{multline*}
f_2(x; T, p, v) = \relu(x\cdot v - (T-p/2))\\ + \relu(x\cdot v - (T+p/2)) - 2\relu(x\cdot v - T).
\end{multline*}
As before, this will only be nonzero in the range $[T-p/2, T+p/2]$, and we are done.

\end{proof}

%% file: prototypes.tex
\section{Prototypicality and Fidelity}
\label{app:prototypical}
\begin{figure}
    \centering
    \includegraphics[width=\linewidth]{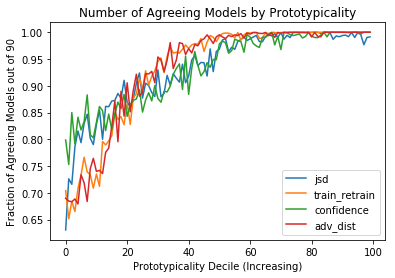}
    \caption{Fidelity is easier on more prototypical examples.}
    \label{fig:prototypes}
\end{figure}

We know from Section~\ref{sec:fid_limits} that learning strategies struggle to achieve perfect fidelity due to non-determinism inherent in learning. What remains to be understood is whether some samples are more difficult than others to achieve fidelity on. We investigate using recent work on identifying prototypical data points. Using each metric developed in Carlini \emph{et al.}~\cite{carlini2019prototypical}, we can rank the Fashion-MNIST test set in order of increasing prototypicality. Binning the prototypicality ranking into percentiles, we can measure how many of the 90 models we trained for Section~\ref{sec:fid_limits} agree with the oracle's prediction. The intuition here is that more prototypical examples should be more consistently learnable, whereas more outlying points may be harder to consistently classify. Indeed, we find that this is the case - all metrics find a correlation between prototypicality and model agreement (fidelity), as seen in Figure~\ref{fig:prototypes}. Interestingly, the metrics which do not use ensembles of models (adversarial distance and holdout-retraining) have the best correlation with the model agreement metric---roughly the top 50\% of prototypical examples by these metrics are classified the same by nearly all 90 models.

%% file: dna_supp.tex
\section{Supplement for Section~\ref{sec:dna}}

\label{app:dna_supp}

Accuracies for the oracles in Section~\ref{sec:dna} are found in Table~\ref{tab:fe_oracle_acc}.

\begin{table}[h!]
\centering
\begin{tabular}{rr|rr}
    \multicolumn{2}{c}{MNIST} & \multicolumn{2}{c}{CIFAR-10} \\
    \toprule
    Parameters &  Accuracy & Parameters & Accuracy \\
    \midrule
    12,500 & 94.3\% & 49,000 & 29.2\% \\
    25,000 & 95.6\% & 98,000 & 34.2\%\\
    50,000 & 97.2\% & 196,000 & 40.3\%\\
    100,000 & 97.7\% & 393,000 & 42.6\% \\
    200,000 & 98.0\% & 786,000 & 43.1\% \\
    400,000 & 98.3\% & 1,572,000 & 45.9\% \\
    \bottomrule
\end{tabular}
\caption{Statistics for the oracle models we train to extract.}
\label{tab:fe_oracle_acc}
\end{table}

Figure~\ref{fig:fi_logits} shows a distribution over the bits of
precision in the difference between the logits (i.e., pre-softmax prediction)
of the 16 neuron oracle neural network and the extracted network. Formally, we measure the magnitude of the gap $|f_\theta(x) - f_{\hat{\theta}}(x)|$. Notice that this is a different (and typically stronger) measure of fidelity than used elsewhere in the paper. 

\begin{figure}
    \centering
    \includegraphics[scale=.8]{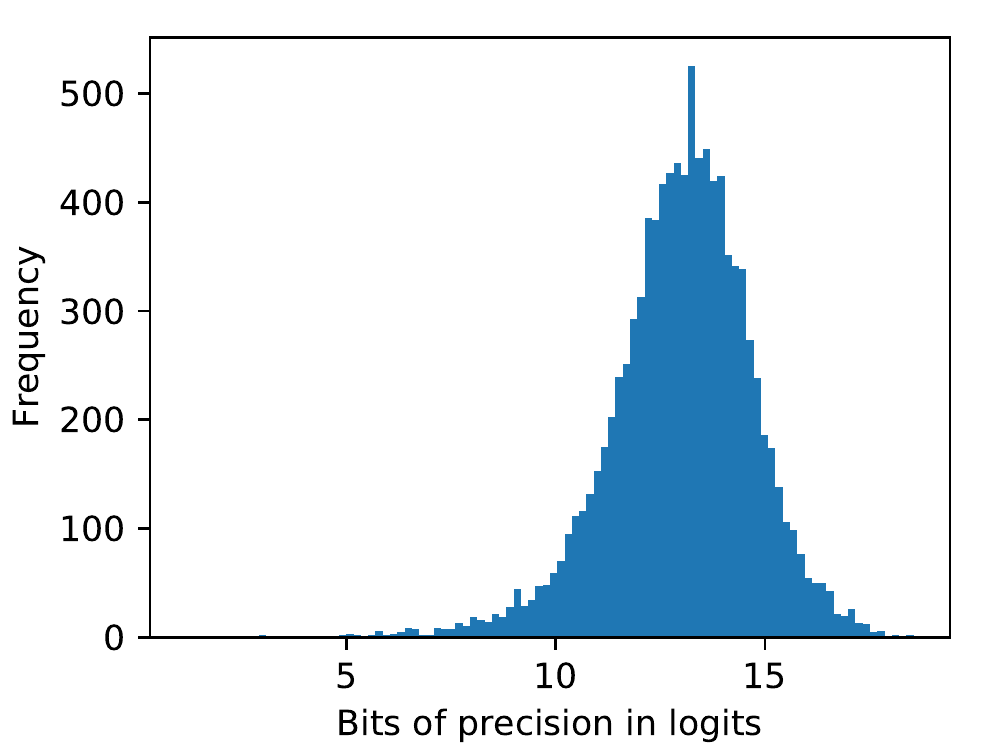}
    \caption{For a 16-neuron MNIST model the attack works. Plotted here is number of bits of precision on the logits normalized by the value of the lot as done in the prior figure.}
    \label{fig:fi_logits}
\end{figure}

\section{Query Complexity of Functionally Equivalent Extraction}
\label{app:dna_queries}
In this section, we briefly analyze the query complexity of the attack from Section~\ref{sec:dna}. We assume that a simulated partial derivative requires $O(1)$ queries using finite differences.
\begin{enumerate}
    \item Critical Point Search. This step is the most nontrivial to analyze, but fortunately this was addressed in~\cite{milli2018model}. They found this step requires $O(h\log(h))$ gradient queries, which we simulate with $O(h\log(h))$ model queries.
    \item Weight Recovery. This piece is significantly complicated by not having access to gradient queries. For each $\relu$, absolute value recovery requires $O(d)$ queries and weight sign recovery requires an additional $O(d)$, making this step take $O(dh)$ queries total.
    \item Global Sign Recovery. For each $\relu$, we require only three queries. Then this step is $O(h)$.
    \item Last Layer Extraction. This step requires $h$ queries to make the system of linear equations full rank (although in practice we reuse previous queries here, making this step require 0 queries).
\end{enumerate}

Overall, the algorithm requires $O(h\log(h)+dh+h)=O(dh)$ queries. Extraction requires $\Omega(dh)$ queries without auxillary information, as there are $dh$ parameters in the model. Then the algorithm is query-optimal up to a constant factor, removing logarithmic factors from Milli \emph{et al.}~\cite{milli2018model}.

%% file: main.bbl
\begin{thebibliography}{10}
\providecommand{\url}[1]{#1}
\csname url@samestyle\endcsname
\providecommand{\newblock}{\relax}
\providecommand{\bibinfo}[2]{#2}
\providecommand{\BIBentrySTDinterwordspacing}{\spaceskip=0pt\relax}
\providecommand{\BIBentryALTinterwordstretchfactor}{4}
\providecommand{\BIBentryALTinterwordspacing}{\spaceskip=\fontdimen2\font plus
\BIBentryALTinterwordstretchfactor\fontdimen3\font minus
  \fontdimen4\font\relax}
\providecommand{\BIBforeignlanguage}[2]{{%
\expandafter\ifx\csname l@#1\endcsname\relax
\typeout{** WARNING: IEEEtran.bst: No hyphenation pattern has been}%
\typeout{** loaded for the language `#1'. Using the pattern for}%
\typeout{** the default language instead.}%
\else
\language=\csname l@#1\endcsname
\fi
#2}}
\providecommand{\BIBdecl}{\relax}
\BIBdecl

\bibitem{strubell2019energy}
E.~Strubell, A.~Ganesh, and A.~McCallum, ``Energy and policy considerations for
  deep learning in nlp,'' \emph{arXiv preprint arXiv:1906.02243}, 2019.

\bibitem{yang2019xlnet}
Z.~Yang, Z.~Dai, Y.~Yang, J.~Carbonell, R.~R. Salakhutdinov, and Q.~V. Le,
  ``Xlnet: Generalized autoregressive pretraining for language understanding,''
  in \emph{Advances in neural information processing systems}, 2019, pp.
  5754--5764.

\bibitem{halevy2009unreasonable}
A.~Halevy, P.~Norvig, and F.~Pereira, ``The unreasonable effectiveness of
  data,'' 2009.

\bibitem{deng2009imagenet}
J.~Deng, W.~Dong, R.~Socher, L.-J. Li, K.~Li, and L.~Fei-Fei, ``Imagenet: A
  large-scale hierarchical image database,'' in \emph{2009 IEEE conference on
  computer vision and pattern recognition}.\hskip 1em plus 0.5em minus
  0.4em\relax Ieee, 2009, pp. 248--255.

\bibitem{sutskever2014sequence}
I.~Sutskever, O.~Vinyals, and Q.~V. Le, ``Sequence to sequence learning with
  neural networks,'' in \emph{Neural information processing systems}, 2014, pp.
  3104--3112.

\bibitem{van2016wavenet}
A.~Van Den~Oord, S.~Dieleman, H.~Zen, K.~Simonyan, O.~Vinyals, A.~Graves,
  N.~Kalchbrenner, A.~W. Senior, and K.~Kavukcuoglu, ``Wavenet: A generative
  model for raw audio.'' \emph{SSW}, vol. 125, 2016.

\bibitem{papernot2017practical}
N.~Papernot, P.~McDaniel, I.~Goodfellow, S.~Jha, Z.~B. Celik, and A.~Swami,
  ``Practical black-box attacks against machine learning,'' in
  \emph{Proceedings of the 2017 ACM on Asia conference on computer and
  communications security}.\hskip 1em plus 0.5em minus 0.4em\relax ACM, 2017,
  pp. 506--519.

\bibitem{lowd2005adversarial}
D.~Lowd and C.~Meek, ``Adversarial learning,'' in \emph{Proceedings of the
  eleventh ACM SIGKDD international conference on Knowledge discovery in data
  mining}.\hskip 1em plus 0.5em minus 0.4em\relax ACM, 2005, pp. 641--647.

\bibitem{shokri2017membership}
R.~Shokri, M.~Stronati, C.~Song, and V.~Shmatikov, ``Membership inference
  attacks against machine learning models,'' in \emph{2017 IEEE Symposium on
  Security and Privacy (SP)}.\hskip 1em plus 0.5em minus 0.4em\relax IEEE,
  2017, pp. 3--18.

\bibitem{salem2018ml}
A.~Salem, Y.~Zhang, M.~Humbert, P.~Berrang, M.~Fritz, and M.~Backes,
  ``Ml-leaks: Model and data independent membership inference attacks and
  defenses on machine learning models,'' \emph{arXiv preprint
  arXiv:1806.01246}, 2018.

\bibitem{tramer2016stealing}
F.~Tram{\`e}r, F.~Zhang, A.~Juels, M.~K. Reiter, and T.~Ristenpart, ``Stealing
  machine learning models via prediction apis,'' in \emph{25th $\{$USENIX$\}$
  Security Symposium ($\{$USENIX$\}$ Security 16)}, 2016, pp. 601--618.

\bibitem{orekondy2019knockoff}
T.~Orekondy, B.~Schiele, and M.~Fritz, ``Knockoff nets: Stealing functionality
  of black-box models,'' in \emph{Proceedings of the IEEE Conference on
  Computer Vision and Pattern Recognition}, 2019, pp. 4954--4963.

\bibitem{DBLP:journals/corr/abs-1811-02054}
\BIBentryALTinterwordspacing
V.~Chandrasekaran, K.~Chaudhuri, I.~Giacomelli, S.~Jha, and S.~Yan, ``Model
  extraction and active learning,'' \emph{CoRR}, vol. abs/1811.02054, 2018.
  [Online]. Available: \url{http://arxiv.org/abs/1811.02054}
\BIBentrySTDinterwordspacing

\bibitem{oh2017towards}
S.~J. Oh, M.~Augustin, B.~Schiele, and M.~Fritz, ``Towards reverse-engineering
  black-box neural networks,'' \emph{arXiv preprint arXiv:1711.01768}, 2017.

\bibitem{DBLP:journals/corr/abs-1905-09165}
\BIBentryALTinterwordspacing
S.~Pal, Y.~Gupta, A.~Shukla, A.~Kanade, S.~K. Shevade, and V.~Ganapathy, ``A
  framework for the extraction of deep neural networks by leveraging public
  data,'' \emph{CoRR}, vol. abs/1905.09165, 2019. [Online]. Available:
  \url{http://arxiv.org/abs/1905.09165}
\BIBentrySTDinterwordspacing

\bibitem{correia2018copycat}
J.~R. Correia-Silva, R.~F. Berriel, C.~Badue, A.~F. de~Souza, and
  T.~Oliveira-Santos, ``Copycat cnn: Stealing knowledge by persuading
  confession with random non-labeled data,'' in \emph{2018 International Joint
  Conference on Neural Networks (IJCNN)}.\hskip 1em plus 0.5em minus
  0.4em\relax IEEE, 2018.

\bibitem{song2019overlearning}
C.~Song and V.~Shmatikov, ``Overlearning reveals sensitive attributes,''
  \emph{arXiv preprint arXiv:1905.11742}, 2019.

\bibitem{hong2018security}
S.~Hong, M.~Davinroy, Y.~Kaya, S.~N. Locke, I.~Rackow, K.~Kulda,
  D.~Dachman-Soled, and T.~Dumitra{\c{s}}, ``Security analysis of deep neural
  networks operating in the presence of cache side-channel attacks,''
  \emph{arXiv preprint arXiv:1810.03487}, 2018.

\bibitem{milli2018model}
S.~Milli, L.~Schmidt, A.~D. Dragan, and M.~Hardt, ``Model reconstruction from
  model explanations,'' \emph{arXiv preprint arXiv:1807.05185}, 2018.

\bibitem{nair2010rectified}
V.~Nair and G.~E. Hinton, ``Rectified linear units improve restricted boltzmann
  machines,'' in \emph{Proceedings of the 27th international conference on
  machine learning (ICML-10)}, 2010, pp. 807--814.

\bibitem{nesterov1983method}
Y.~E. Nesterov, ``A method for solving the convex programming problem with
  convergence rate o (1/k\^{} 2),'' in \emph{Dokl. akad. nauk Sssr}, vol. 269,
  1983, pp. 543--547.

\bibitem{duchi2011adaptive}
J.~Duchi, E.~Hazan, and Y.~Singer, ``Adaptive subgradient methods for online
  learning and stochastic optimization,'' \emph{Journal of Machine Learning
  Research}, vol.~12, no. Jul, pp. 2121--2159, 2011.

\bibitem{kingma2014adam}
D.~P. Kingma and J.~Ba, ``Adam: A method for stochastic optimization,''
  \emph{arXiv preprint arXiv:1412.6980}, 2014.

\bibitem{hinton2015distilling}
G.~Hinton, O.~Vinyals, and J.~Dean, ``Distilling the knowledge in a neural
  network,'' \emph{arXiv preprint arXiv:1503.02531}, 2015.

\bibitem{batina2018csi}
L.~Batina, S.~Bhasin, D.~Jap, and S.~Picek, ``Csi neural network: Using
  side-channels to recover your artificial neural network information,''
  \emph{arXiv preprint arXiv:1810.09076}, 2018.

\bibitem{kocher1999differential}
P.~Kocher, J.~Jaffe, and B.~Jun, ``Differential power analysis,'' in
  \emph{Annual International Cryptology Conference}.\hskip 1em plus 0.5em minus
  0.4em\relax Springer, 1999, pp. 388--397.

\bibitem{DBLP:journals/corr/abs-1904-03866}
A.~Das, S.~Gollapudi, R.~Kumar, and R.~Panigrahy, ``On the learnability of deep
  random networks,'' \emph{CoRR}, vol. abs/1904.03866, 2019.

\bibitem{mahajan2018exploring}
D.~Mahajan, R.~Girshick, V.~Ramanathan, K.~He, M.~Paluri, Y.~Li, A.~Bharambe,
  and L.~van~der Maaten, ``Exploring the limits of weakly supervised
  pretraining,'' in \emph{Proceedings of the European Conference on Computer
  Vision (ECCV)}, 2018, pp. 181--196.

\bibitem{micaelli2019zero}
P.~Micaelli and A.~Storkey, ``Zero-shot knowledge transfer via adversarial
  belief matching,'' \emph{arXiv preprint arXiv:1905.09768}, 2019.

\bibitem{radford2019language}
A.~Radford, J.~Wu, R.~Child, D.~Luan, D.~Amodei, and I.~Sutskever, ``Language
  models are unsupervised multitask learners,'' \emph{OpenAI Blog}, vol.~1,
  no.~8, 2019.

\bibitem{sharif2014cnn}
A.~Sharif~Razavian, H.~Azizpour, J.~Sullivan, and S.~Carlsson, ``Cnn features
  off-the-shelf: an astounding baseline for recognition,'' in \emph{Proceedings
  of the IEEE conference on computer vision and pattern recognition workshops},
  2014, pp. 806--813.

\bibitem{devlin2018bert}
J.~Devlin, M.-W. Chang, K.~Lee, and K.~Toutanova, ``Bert: Pre-training of deep
  bidirectional transformers for language understanding,'' \emph{arXiv preprint
  arXiv:1810.04805}, 2018.

\bibitem{angluin1988queries}
D.~Angluin, ``Queries and concept learning,'' \emph{Machine learning}, vol.~2,
  no.~4, pp. 319--342, 1988.

\bibitem{blum1998combining}
A.~Blum and T.~Mitchell, ``Combining labeled and unlabeled data with
  co-training,'' in \emph{Proceedings of the eleventh annual conference on
  Computational learning theory}.\hskip 1em plus 0.5em minus 0.4em\relax
  Citeseer, 1998, pp. 92--100.

\bibitem{song2020combining}
\BIBentryALTinterwordspacing
S.~Song, D.~Berthelot, and A.~Rostamizadeh, ``Combining mixmatch and active
  learning for better accuracy with fewer labels,'' 2020. [Online]. Available:
  \url{https://openreview.net/forum?id=HJxWl0NKPB}
\BIBentrySTDinterwordspacing

\bibitem{simeoni2020rethinking}
\BIBentryALTinterwordspacing
O.~Sim{\'e}oni, M.~Budnik, Y.~Avrithis, and G.~Gravier, ``Rethinking deep
  active learning: Using unlabeled data at model training,'' 2020. [Online].
  Available: \url{https://openreview.net/forum?id=rJehllrtDS}
\BIBentrySTDinterwordspacing

\bibitem{zhai2019s}
X.~Zhai, A.~Oliver, A.~Kolesnikov, and L.~Beyer, ``S4l: Self-supervised
  semi-supervised learning,'' \emph{arXiv preprint arXiv:1905.03670}, 2019.

\bibitem{berthelot2019mixmatch}
D.~Berthelot, N.~Carlini, I.~Goodfellow, N.~Papernot, A.~Oliver, and C.~Raffel,
  ``Mixmatch: A holistic approach to semi-supervised learning,'' \emph{arXiv
  preprint arXiv:1905.02249}, 2019.

\bibitem{netzer2011reading}
Y.~Netzer, T.~Wang, A.~Coates, A.~Bissacco, B.~Wu, and A.~Y. Ng, ``Reading
  digits in natural images with unsupervised feature learning,'' 2011.

\bibitem{krizhevsky2009learning}
A.~Krizhevsky \emph{et~al.}, ``Learning multiple layers of features from tiny
  images,'' Citeseer, Tech. Rep., 2009.

\bibitem{sculley2015hidden}
D.~Sculley, G.~Holt, D.~Golovin, E.~Davydov, T.~Phillips, D.~Ebner,
  V.~Chaudhary, M.~Young, J.-F. Crespo, and D.~Dennison, ``Hidden technical
  debt in machine learning systems,'' in \emph{Advances in neural information
  processing systems}, 2015, pp. 2503--2511.

\bibitem{lakshminarayanan2017simple}
B.~Lakshminarayanan, A.~Pritzel, and C.~Blundell, ``Simple and scalable
  predictive uncertainty estimation using deep ensembles,'' in \emph{Advances
  in Neural Information Processing Systems}, 2017, pp. 6402--6413.

\bibitem{xiao2017/online}
H.~Xiao, K.~Rasul, and R.~Vollgraf. (2017) Fashion-mnist: a novel image dataset
  for benchmarking machine learning algorithms.

\bibitem{carlini2019prototypical}
\BIBentryALTinterwordspacing
N.~Carlini, U.~Erlingsson, and N.~Papernot, ``Prototypical examples in deep
  learning: Metrics, characteristics, and utility,'' 2019. [Online]. Available:
  \url{https://openreview.net/forum?id=r1xyx3R9tQ}
\BIBentrySTDinterwordspacing

\bibitem{lecun1998gradient}
Y.~LeCun, L.~Bottou, Y.~Bengio, P.~Haffner \emph{et~al.}, ``Gradient-based
  learning applied to document recognition,'' \emph{Proceedings of the IEEE},
  vol.~86, no.~11, pp. 2278--2324, 1998.

\bibitem{JAX}
Google, ``Jax,'' https://github.com/google/jax, 2019.

\bibitem{szegedy2013intriguing}
C.~Szegedy, W.~Zaremba, I.~Sutskever, J.~Bruna, D.~Erhan, I.~Goodfellow, and
  R.~Fergus, ``Intriguing properties of neural networks,'' \emph{arXiv preprint
  arXiv:1312.6199}, 2013.

\bibitem{madry2017towards}
A.~Madry, A.~Makelov, L.~Schmidt, D.~Tsipras, and A.~Vladu, ``Towards deep
  learning models resistant to adversarial attacks,'' \emph{arXiv preprint
  arXiv:1706.06083}, 2017.

\bibitem{lee2018defending}
T.~Lee, B.~Edwards, I.~Molloy, and D.~Su, ``Defending against model stealing
  attacks using deceptive perturbations,'' \emph{arXiv preprint
  arXiv:1806.00054}, 2018.

\bibitem{alabdulmohsin2014adding}
I.~M. Alabdulmohsin, X.~Gao, and X.~Zhang, ``Adding robustness to support
  vector machines against adversarial reverse engineering,'' in
  \emph{Proceedings of the 23rd ACM International Conference on Conference on
  Information and Knowledge Management}.\hskip 1em plus 0.5em minus 0.4em\relax
  ACM, 2014, pp. 231--240.

\bibitem{juuti2018prada}
M.~Juuti, S.~Szyller, A.~Dmitrenko, S.~Marchal, and N.~Asokan, ``Prada:
  protecting against dnn model stealing attacks,'' \emph{arXiv preprint
  arXiv:1805.02628}, 2018.

\bibitem{kesarwani2018model}
M.~Kesarwani, B.~Mukhoty, V.~Arya, and S.~Mehta, ``Model extraction warning in
  mlaas paradigm,'' in \emph{Proceedings of the 34th Annual Computer Security
  Applications Conference}.\hskip 1em plus 0.5em minus 0.4em\relax ACM, 2018,
  pp. 371--380.

\bibitem{wang2018stealing}
B.~Wang and N.~Z. Gong, ``Stealing hyperparameters in machine learning,'' in
  \emph{2018 IEEE Symposium on Security and Privacy (SP)}.\hskip 1em plus 0.5em
  minus 0.4em\relax IEEE, 2018, pp. 36--52.

\bibitem{zhang2018protecting}
J.~Zhang, Z.~Gu, J.~Jang, H.~Wu, M.~P. Stoecklin, H.~Huang, and I.~Molloy,
  ``Protecting intellectual property of deep neural networks with
  watermarking,'' in \emph{Proceedings of the 2018 on Asia Conference on
  Computer and Communications Security}.\hskip 1em plus 0.5em minus 0.4em\relax
  ACM, 2018, pp. 159--172.

\bibitem{uchida2017embedding}
Y.~Uchida, Y.~Nagai, S.~Sakazawa, and S.~Satoh, ``Embedding watermarks into
  deep neural networks,'' in \emph{Proceedings of the 2017 ACM on International
  Conference on Multimedia Retrieval}.\hskip 1em plus 0.5em minus 0.4em\relax
  ACM, 2017, pp. 269--277.

\bibitem{barak2001possibility}
B.~Barak, O.~Goldreich, R.~Impagliazzo, S.~Rudich, A.~Sahai, S.~Vadhan, and
  K.~Yang, ``On the (im) possibility of obfuscating programs,'' in \emph{Annual
  international cryptology conference}.\hskip 1em plus 0.5em minus 0.4em\relax
  Springer, 2001, pp. 1--18.

\bibitem{barni2006privacy}
M.~Barni, C.~Orlandi, and A.~Piva, ``A privacy-preserving protocol for
  neural-network-based computation,'' in \emph{Proceedings of the 8th workshop
  on Multimedia and security}.\hskip 1em plus 0.5em minus 0.4em\relax ACM,
  2006, pp. 146--151.

\bibitem{katz2017reluplex}
G.~Katz, C.~Barrett, D.~L. Dill, K.~Julian, and M.~J. Kochenderfer, ``Reluplex:
  An efficient smt solver for verifying deep neural networks,'' in
  \emph{International Conference on Computer Aided Verification}.\hskip 1em
  plus 0.5em minus 0.4em\relax Springer, 2017, pp. 97--117.

\end{thebibliography}
